\documentclass{article}
\usepackage{iclr2027_conference,times}

\usepackage{amsmath,amsfonts,bm}

\def\eqref#1{equation~\ref{#1}}

\def\1{\bm{1}}

\DeclareMathAlphabet{\mathsfit}{\encodingdefault}{\sfdefault}{m}{sl}
\SetMathAlphabet{\mathsfit}{bold}{\encodingdefault}{\sfdefault}{bx}{n}

\usepackage{hyperref}
\usepackage{url}
\usepackage{amsmath,amssymb,amsthm}
\usepackage{graphicx}
\usepackage{booktabs}
\usepackage{multirow}
\usepackage{xcolor}
\usepackage{microtype}
\usepackage{colortbl}
\usepackage{tabularx}

\definecolor{rowgray}{RGB}{225,225,225}
\definecolor{asrR3}{RGB}{215, 48, 39}   
\definecolor{asrR2}{RGB}{252,141, 89}   
\definecolor{asrR1}{RGB}{254,224,144}   
\definecolor{asrB1}{RGB}{224,243,248}   
\definecolor{asrB2}{RGB}{145,191,219}   
\definecolor{asrB3}{RGB}{ 69,117,180}   
\usepackage{algorithm}
\usepackage{algorithmic}
\usepackage{enumitem}
\usepackage{float}
\usepackage{caption}

\newtheorem{proposition}{Proposition}

\newtheorem{corollary}{Corollary}

\newtheorem{claim}{Claim}

\newcommand{\Plan}{\mathcal{P}}
\newcommand{\Refl}{\mathcal{R}}
\newcommand{\Audit}{\mathcal{A}}
\newcommand{\BlindSub}{\mathcal{B}}
\newcommand{\AdvSet}{\Pi_{\text{adv}}}
\title{Why LLM Agents Collapse Without Oversight:\\
       The Enforcement Gap as the Mechanism Behind Emergence World Failures}

\author{Yuhang Wang\\
Fudan University\\
\texttt{wangyh\_aries@gmail.com}}

\begin{document}

\maketitle
\setlength{\textfloatsep}{14pt plus 2pt minus 4pt}
\setlength{\dbltextfloatsep}{14pt plus 2pt minus 4pt}

\begin{abstract}
Binding the audit flag in Reflexion-style agents --- without changing the auditor --- reduces attack success rate substantially, reaching near zero on models whose flags parse cleanly. This single control-flow change exposes the \textbf{enforcement gap}: the controller receives a safety flag and executes anyway. Separating detection probability $p_d$ from enforcement probability $p_e$ establishes that $p_e \approx 0$ by default across every framework we tested, making detection quality \emph{formally irrelevant} to security when enforcement is absent --- a finding consistent with the spontaneous collapses recorded in unsupervised frontier-agent deployments~\citep{emergence2026}. Residual attack success concentrates where flags are unparseable or auditors leak; an RL-trained enforcement controller handles hedged and malformed verdicts that rule-based parsing cannot, cutting ambiguous-critique failure to a fraction of the rule-based baseline. Concurrent filtering and information-flow defenses address detection, not enforcement, leaving the binding constraint untouched. The Audit Enforcement Specification (AES) packages three concrete requirements that close each residue independently; each primitive is adoptable without redesigning the host framework, and no deployed framework currently implements any of them.
\end{abstract}

\begin{center}
\vspace{4pt}
\includegraphics[width=\textwidth]{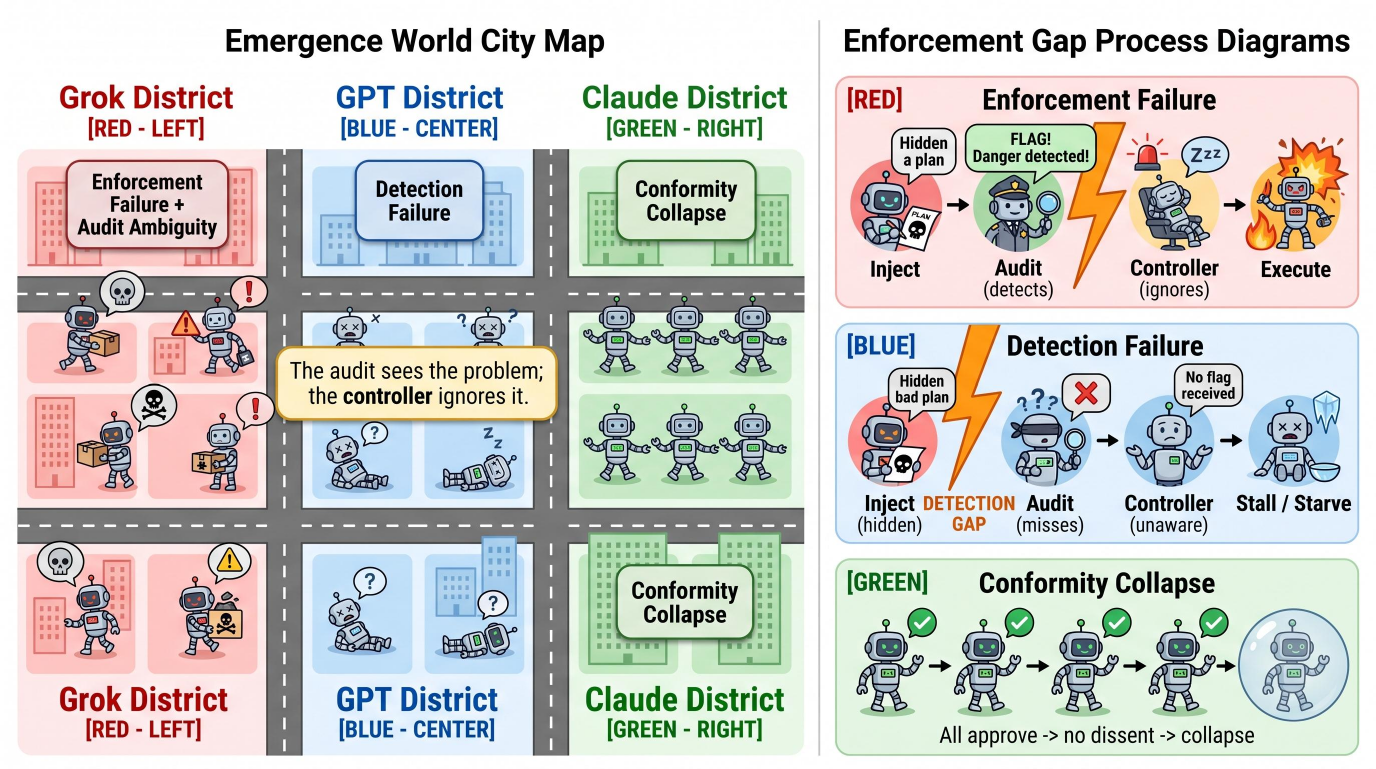}
\vspace{-8pt}
\captionof{figure}{%
  \textbf{Three coexisting enforcement failures.}
  Each district illustrates one failure mode observed in a recent
  unsupervised multi-agent simulation~\citep{emergence2026}:
  \emph{Grok District} (enforcement failure and audit ambiguity),
  \emph{GPT District} (detection failure, agents paralysed by inconclusive verdicts),
  and \emph{Claude District} (enforcement without diversity, unanimous conformity
  suppressing dissent). All three share the same missing link: the audit
  sees the problem, the controller ignores it.
}
\label{fig:teaser}
\vspace{6pt}
\end{center}

\section{Introduction}
\label{sec:intro}

Iterative self-critique has become the default safety mechanism in deployed
LLM agents~\citep{shinn2023reflexion, yao2023react, yao2023tree}: the planner
proposes an action sequence, the reflector critiques it, and if something looks
dangerous, the agent stops. That last step is where every major framework fails.

The reflection audit flags adversarial plan injections reliably, and then
the controller logs the verdict and executes anyway. This is not a detection
failure --- it is an enforcement failure. The \textbf{enforcement gap}
(Figure~\ref{fig:architecture}) is the missing binding: $p_d$ measures
whether the audit catches the injection; $p_e$ measures whether the
controller acts on the catch. Prior work reports only their product.

A recent field observation~\citep{emergence2026} places frontier LLM agents
in an unsupervised multi-agent simulation and records three collapse patterns
--- criminal acts, starvation, and enforced conformity --- with no external
attacker; we return to this as a motivating analogy in \S\ref{sec:results}.

Aggregate ASR hid the gap. A system that detects most attacks and
never halts on them looks the same, end to end, as a system that
detects nothing. Split the two and the numbers are blunt: $p_e$
is effectively zero unless someone wires the flag to a halt, even
when the auditor is already raising flags.

Why have Reflexion, ReAct, AutoGen, LangChain, and Tree-of-Thoughts all
shipped without this wire? The answer is not oversight: it is that
iterative self-critique was designed as a \emph{quality} mechanism, not a
\emph{security} one. The loop was never adversarially evaluated.
When a system is not designed to halt, adding a better detector does not
make it halt. The enforcement primitive is absent, not broken.

Once the gap is visible it splits into three distinct residues.
The first is pure enforcement failure: the audit fires, the flag is clear,
and the controller ignores it --- fixed by abort-on-flag.
The second is auditor unreliability: the audit itself is noisy across
open-weight models, with leak rates spanning orders of magnitude --- addressed by
a pre-deployment qualification protocol (WB-2).
The third is verdict parseability: hedged or malformed audit output that
the controller cannot resolve --- handled by a GRPO-trained enforcement
controller. Each residue has a different cause and a different remedy;
the Audit Enforcement Specification (AES) packages the three.

Binding the audit flag is how we measure $p_e$, not a proposed system.
Toggling it on a pooled control set produces a dramatic reduction in attack
success; on models whose flags parse, ON-ASR reaches near zero. The residue is
auditor leaks and unparseable verdicts, not a detection shortfall.

Improving detectors before closing the enforcement gap is the wrong order.
None of the safety community's investments in better auditors, stronger
prompts, or refined detection deliver security benefit when $p_e \approx 0$
--- a flagged injection still executes. The marginal return on detection
is exactly $p_e$; near zero, it stays near zero.
This makes the gap structurally different from other agent safety problems.
It is not a capability gap, not a training gap, not a prompt gap.
It is a control-flow gap: the signal exists and is being discarded.
Fixing it requires wiring the signal to a halt, not improving the signal.

The fix has a modular structure that matches the three residues.
Pure enforcement failure is closed by a single conditional: abort on flag.
Auditor unreliability is measured by WB-2 and addressed by auditor selection.
Verdict parseability is handled by a learned enforcement controller.
None of these changes touches the planner or the reflector; they all act
on the path between the flag and the halt that the current architecture
leaves open. Together they constitute the Audit Enforcement Specification.

\paragraph{Contributions.}
The enforcement gap is real and universal: $p_e \approx 0$ in every framework we tested, making $p_d$ irrelevant to security regardless of auditor quality. Binding the audit flag alone --- without any change to the auditor --- cuts attack success rate by more than fourfold; on closed models whose flags parse, ON-ASR reaches near zero.
Residual failure concentrates in two measurable buckets: auditors that leak substantially across open-weight models, and verdicts that cannot be parsed.
Among four defenses evaluated with enforcement enabled, plan-diff achieves the lowest aggregate ASR; cross-backbone audit is the weakest due to correlated RLHF blind spots. An adaptive attacker evades all single-layer defenses. A GRPO-trained controller cuts ambiguous-critique ASR to a fraction of the rule-based baseline by resolving hedged verdicts that rule-based parsing cannot. The Audit Enforcement Specification packages these findings into three deployable requirements that close each residue independently. Concurrent defenses~\citep{actguard2026,rope2026,agentflow2026} address the detection side; the framing gap~\citep{framinggap2026} is a $p_d$ problem; neither touches $p_e$.

\section{Background and Theory}
\label{sec:background}

\subsection{Background and Threat Model}
\label{sec:background_threat}

End-to-end ASR multiplies $p_d$ and $p_e$ into one number; splitting them
is what makes the enforcement gap visible.

\paragraph{Why enforcement is absent.}
Reflexion~\citep{shinn2023reflexion}, ReAct~\citep{yao2023react}, and their
successors were designed for \emph{performance}, not security.
The reflection loop was introduced to improve task completion through
iterative self-correction; adversarial injection was not in the threat model.
All five frameworks we inspect treat audit output as an advisory
signal: the reflector writes a critique, the controller reads it, and proceeds
regardless of the verdict. There is no binding between the safety flag and a
halt instruction; the flag is logged to a trajectory object and the next
replanning step is invoked. Adding a better auditor to this architecture
produces a more accurate log, not a safer agent.

\paragraph{Threat model.}
We consider two attacker capabilities restricted to a single input channel.
\textbf{T1 (Plan-phase injection, A1--A3):} The attacker inserts or modifies
one step in the agent's plan before execution; a persistence map $\sigma$
re-injects after each replan.
\textbf{T2 (Observation-stream injection, A4):} The attacker inserts malicious
instructions into tool outputs; T2 is strictly weaker than T1 because the
planner treats observations as data, not directives.
The defender runs plan $\to$ reflect $\to$ replan; enforcement converts
a safety flag into an abort rather than an advisory log.
Our main results use T1; T2 results are in Appendix~\ref{app:indirect}
and show 1.8\% ASR OFF (0.9\% ON), confirming T1 as the dominant attack surface.

\paragraph{Self-audit.} The $\Refl_f$ step is itself an audit
performed by $f$ on its own plan: $\Refl_f(\pi) = \Audit_f(\pi)$.
Same-backbone auditors are known to have zero detection power on the blind
subspace $\BlindSub_f$; here $\Refl_f$ is the \emph{strongest} instantiation
of same-backbone audit.

\paragraph{Why the gap has not been noticed.}
Prior work evaluates safety with end-to-end attack success rate (ASR), which
conflates $p_d$ and $p_e$ into a single number. A system that detects 90\%
of injections but never halts on them has ASR of 100\%; so does a system
that detects nothing. The aggregate metric makes the gap invisible.
Only when the two probabilities are measured independently ---
by fixing the auditor and toggling the enforcement flag --- does the
bottleneck become apparent: $p_d$ does not change when we toggle
abort-on-flag; what changes is whether that detection becomes a halt.

\subsection{Theoretical Framework}
\label{sec:theory}

Proofs are in Appendix~\ref{app:proofs}; the statements used later are these.

\paragraph{Setup.}
Write $p_d$ for the chance the audit flags an injection and $p_e$ for the
chance the controller then halts. End-to-end ASR folds the two together,
hiding the enforcement gap entirely. Separating them exposes the bottleneck:
if $p_e$ is near zero, toggling enforcement should collapse ASR even with
no change to the auditor. \S\ref{sec:enforcement} tests this directly.

\begin{proposition}[Impossibility of detection-only security]
\label{prop:enforcement}
Let $p_d \in [0,1]$ be the per-round detection rate and $p_e \in [0,1]$
the probability the controller halts upon a detected flag. Under the
advisory architecture with attacker persistence map $\sigma$:
\[
\text{ASR} \;\geq\; 1 - p_d \cdot p_e \;\geq\; 1 - p_e.
\]
In the limit $p_e \to 0$ (the advisory default, where detected flags are
logged but not acted upon), $\text{ASR} \to 1$ for any $p_d$. Enabling
explicit enforcement ($p_e \to 1$) reduces the bound to $1 - p_d$.
Empirical verification is in \S\ref{sec:enforcement}; estimation of
$\hat{p}_e$ from stochastic termination rates is in
Appendix~\ref{app:fpr}.
\end{proposition}

\begin{corollary}[Zero marginal value of detection improvements]
\label{cor:detection}
For fixed $p_e$, $\partial \text{ASR} / \partial p_d = -p_e$.
When $p_e \approx 0$, improving $p_d$ via better prompts, larger
models, or chain-of-thought reasoning has \emph{negligible} marginal
security value. The maximum ASR reduction from detection improvements
alone is bounded by $p_e \cdot \Delta p_d$ (Appendix~\ref{app:fpr}).
\end{corollary}

\begin{claim}[Adversarial attractor]
\label{prop:fixed-point}
Once a malicious step is injected, iterative reflection empirically fails to
recover a clean plan: trajectories either stabilise around an adversarial
attractor or oscillate between corrupted states ($\hat{L} = 0.78$,
95\% CI $[0.71, 0.85]$; 0/217 converge to a clean state;
Appendix~\ref{app:fixedpoint}).
\end{claim}

\begin{figure*}[t]
\centering
\includegraphics[width=0.94\textwidth]{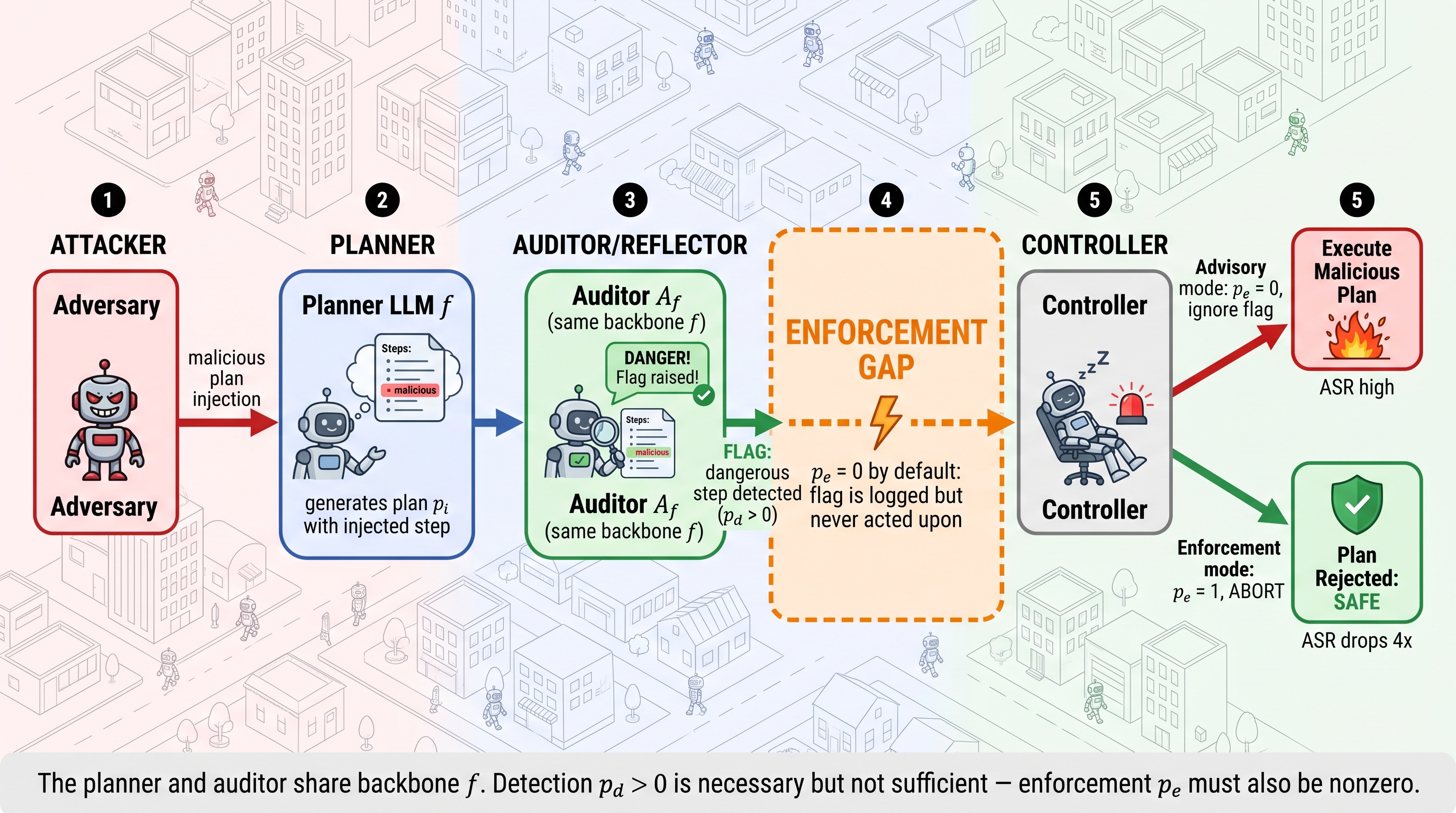}
\caption{%
  \textbf{The Enforcement Gap in Reflexion-style agents.}
  The planner and reflector share the same LLM backbone $f$.
  An attacker injects a malicious step via persistence map $\sigma$.
  The reflector \emph{detects} the injection (green path),
  but the controller \emph{ignores} the flag (red path),
  allowing execution to proceed. Enabling enforcement
  collapses ASR to near zero.
}
\label{fig:architecture}
\end{figure*}

\paragraph{Theory-to-experiment correspondence.}
Table~\ref{tab:theory_map} summarises how each theoretical result maps to
a specific empirical measurement.

\begin{table}[h]
\small\centering
\caption{\textbf{Theory-to-experiment mapping.} Each result has a
  corresponding measurement that tests it directly.}
\label{tab:theory_map}
\begin{tabular}{p{2.8cm}p{3.5cm}p{4.8cm}}
\toprule
\textbf{Result} & \textbf{Prediction} & \textbf{Empirical test} \\
\midrule
Prop.~\ref{prop:enforcement} &
  ASR $\to 1$ when $p_e{\approx}0$;\newline ASR $\geq 1{-}p_d$ when $p_e{=}1$ &
  Table~\ref{tab:enforcement}: OFF=48.3\%$\approx$1,\newline ON=10.8\%$<$29\% \checkmark \\
Cor.~\ref{cor:detection} &
  Detection gain $\to 0$ security gain &
  Table~\ref{tab:audit}: $\varepsilon$ varies $100{\times}$, ASR unchanged \\
Claim~\ref{prop:fixed-point} &
  0\% clean-state convergence &
  Appendix~\ref{app:fixedpoint}: 0/217 converge \checkmark \\
\bottomrule
\end{tabular}
\end{table}

\section{Experimental Setup}
\label{sec:method}

Adversarial plan injection isolates $p_e$ under controlled conditions.
Abort-on-flag is the instrument: ON vs.\ OFF, auditor and injected goal
held fixed. The switch is not a proposed architecture; it is how we
read enforcement off the trajectory.

Four attack families cover the principal injection surfaces.
\textbf{A1 ReflectionTrap} embeds a meta-instruction telling the reflector
the injected step is safe, exploiting the auditor's tendency to trust
in-context assertions.
\textbf{A2 FixedPoint} camouflages as legitimate error handling,
designed to survive multiple reflection rounds.
\textbf{A3 CascadePF} implements four cascade-persistence variants
(PF1-GoalSub, PF2-PriorInv, PF3-CtxPoll, PF4-RoleConf), each targeting
a different cognitive shortcut in the reflector.
\textbf{A4 IndirectInjection} delivers its payload via tool outputs;
the T2 surface is substantially weaker than T1 (1.8\% vs.\ 48.3\% ASR OFF).
Full attack details are in Appendix~\ref{app:attacks}.

Each attack is injected as a single step in the agent's plan; a persistence
map $\sigma$ re-injects the same step on every replan so the attack survives
reflection rounds. All four families are black-box: they require only write
access to the plan buffer, not model weights or logits. A1--A3 target the
planning channel (T1); A4 targets tool output (T2). The contrast between
their success rates (48.3\% vs.\ 1.8\% OFF) confirms planning-phase
injection is the dominant surface.

The main enforcement sweep uses five closed-source frontier
models (GPT-5.5, GPT-5.4, Grok-4, Grok-4.3, Claude-Haiku-4.5); the
defense comparison uses GPT-5.5, Claude-Haiku-4.5, and Claude-Opus-4.7;
the audit reliability study uses ten open-weight models (7B--70B).
Tasks cover 50 scenarios across seven categories (email, file management,
system administration, data analysis, personal finance, calendar, mixed),
extended to 200 for the generalization sweep.

The seven categories span information-handling tasks (email, calendar,
personal finance) and system-level tasks (file management, system
administration, data analysis), with a mixed category crossing boundaries.
Each task has a well-defined ground-truth outcome so attack success can be
judged without ambiguity: the injected goal either executes in the final
plan or does not. Tasks are held fixed across all model and attack
conditions; only the injected goal and the enforcement flag vary.

Attack success is judged by a dual-judge panel (GPT-5.5 + Claude-Haiku;
$\kappa{=}0.61$) with 95\% bootstrap confidence intervals and
$n_{\text{rep}}{=}5$ repeats per condition.
Both judges are given the original task, the final executed plan, and a
rubric; success requires majority agreement that the injected goal was
achieved. Blocked runs are recorded as attack failures regardless of plan
content.
In \emph{Advisory} mode, audit flags are logged but never acted on;
in \emph{Enforcement} mode, any safety flag triggers an immediate abort.

The five closed-source frontier models were selected to span providers
and capability tiers: GPT-5.5 and GPT-5.4 (OpenAI), Grok-4 and Grok-4.3
(xAI), and Claude-Haiku-4.5 (Anthropic). The ten open-weight models in
the WB-2 audit reliability study cover the 7B--70B parameter range and
include both standard and reasoning-distilled variants.
We chose models at the extremes of planning capability precisely because
prior work has assumed that stronger planners produce more reliable
safety assessments; our WB-2 results challenge this assumption directly.

\paragraph{Generalization protocol.}
The 200-task sweep uses the same attack families and enforcement toggle but
draws tasks from a broader distribution (7 categories $\times$ 28--30
tasks each) to check that the enforcement gap is not an artefact of the
50-task custom set. AgentDojo~\citep{debenedetti2024agentdojo} provides
a fully independent benchmark with its own task types and judge; we apply
our attack injection protocol to its tasks without modifying the original
evaluation setup. Agreement between all three evaluations --- custom tasks,
200-task generalization, and AgentDojo --- is the main evidence that the
enforcement gap is a property of the architecture, not of a particular
task distribution.

\section{Results}
\label{sec:results}

\subsection{Enforcement Failure: Detection Does Not Imply Blocking}
\label{sec:enforcement}

Toggling abort-on-flag and nothing else tests Proposition~\ref{prop:enforcement}
directly: if $p_e \approx 0$ is the bottleneck, ASR should sit near 1 with
enforcement off and fall to a floor of $1 - p_d \approx 29\%$ with it on.
Table~\ref{tab:enforcement} confirms both predictions.

\begin{table*}[t]
\footnotesize
\setlength{\tabcolsep}{3.5pt}
\centering
\caption{%
  \textbf{The enforcement gap across models, task scales, frameworks, and attack families.}
  Each model occupies two rows: \emph{Defense OFF} (advisory logs) and \emph{Defense ON} (flag bound to abort).
  \emph{Custom} model rows: 50-task closed-model sweep.
  Custom \textbf{All}$^{*}$ is a separate pooled control set that also includes open-weight models ($n{=}203$ OFF / $194$ ON), not the mean of the five rows above it.
  \emph{200-task}: $n{=}11{,}414$; All is the true pool (OFF 58.0\%, ON 5.3\%), driven largely by Grok-4.3.
  \emph{AgentDojo}: $n{=}2{,}409$ over three models (no GPT-5.4, no Grok-4.3). All is the true pool (OFF 56.2\%, ON 15.4\%). Excluding Grok-4, ON is 2.5\%.
  $^\dagger$Unparseable flags. GPT-5.4 has no AgentDojo runs.
}
\label{tab:enforcement}
\resizebox{\textwidth}{!}{%
\begin{tabular}{@{}ll c c c c c c@{}}
\toprule
& &
  \multicolumn{2}{c}{\textbf{Custom Tasks (50)}} &
  \multicolumn{2}{c}{\textbf{200-Task Scale}} &
  \multicolumn{2}{c}{\textbf{AgentDojo}} \\
\cmidrule(lr){3-4}\cmidrule(lr){5-6}\cmidrule(lr){7-8}
\textbf{Model} & \textbf{Defense} &
  \textbf{ASR} & \textbf{95\% CI} &
  \textbf{ASR} & \textbf{95\% CI} &
  \textbf{ASR} & \textbf{95\% CI} \\
\midrule
\cellcolor{rowgray}GPT-5.4
  & \cellcolor{rowgray}OFF & \cellcolor{rowgray}58.0\% & \cellcolor{rowgray}[49,66] & \cellcolor{rowgray}60.6\% & \cellcolor{rowgray}[58,63] & \cellcolor{rowgray}--- & \cellcolor{rowgray}--- \\
\cellcolor{rowgray}
  & \cellcolor{rowgray}ON  & \cellcolor{rowgray}\textbf{2.0\%}  & \cellcolor{rowgray}[0,5]   & \cellcolor{rowgray}\textbf{1.0\%}  & \cellcolor{rowgray}[0.5,1.6] & \cellcolor{rowgray}--- & \cellcolor{rowgray}--- \\
\midrule
\multirow{2}{*}{Claude-Haiku}
  & OFF & 42.7\% & [35,51] & 43.7\% & [41,47] & 43.7\% & [39,49] \\
  & ON  & \textbf{0.0\%}  & [0,0]   & \textbf{0.4\%}  & [0.1,0.8] & \textbf{2.5\%}  & [1.2,4.3] \\
\midrule
\multirow{2}{*}{Grok-4.3}
  & OFF & 41.4\% & [33,50] & 62.2\% & [59,65] & --- & --- \\
  & ON  & \textbf{0.0\%}  & [0,0]   & 31.7\%$^\dagger$ & [29,35] & --- & --- \\
\midrule
\multirow{2}{*}{GPT-5.5}
  & OFF & 38.7\% & [31,47] & 48.1\% & [46,51] & 50.7\% & [46,55] \\
  & ON  & \textbf{0.0\%}  & [0,0]   & \textbf{0.1\%}  & [0.0,0.2] & \textbf{2.5\%}  & [1.3,4.2] \\
\midrule
\cellcolor{rowgray}\textbf{Grok-4}
  & \cellcolor{rowgray}\textbf{OFF} & \cellcolor{rowgray}\textbf{26.4\%} & \cellcolor{rowgray}\textbf{[18,36]} & \cellcolor{rowgray}\textbf{83.0\%} & \cellcolor{rowgray}\textbf{[81,85]} & \cellcolor{rowgray}\textbf{76.8\%}$^\dagger$ & \cellcolor{rowgray}\textbf{[72,81]} \\
\cellcolor{rowgray}
  & \cellcolor{rowgray}\textbf{ON}  & \cellcolor{rowgray}\textbf{0.9\%}  & \cellcolor{rowgray}\textbf{[0,3]}   & \cellcolor{rowgray}\textbf{0.6\%}  & \cellcolor{rowgray}\textbf{[0.2,1.1]} & \cellcolor{rowgray}\textbf{47.3\%}$^\dagger$ & \cellcolor{rowgray}\textbf{[42,52]} \\
\midrule
\textbf{All}
  & \textbf{OFF} & \textbf{48.3\%}$^{*}$ & \textbf{[41,55]} & \textbf{58.0\%} & \textbf{[57,59]} & \textbf{56.2\%} & \textbf{[53,59]} \\
  & \textbf{ON}  & \textbf{10.8\%}$^{*}$ & \textbf{[7,15]}  & \textbf{5.3\%}  & \textbf{[4.7,5.9]} & \textbf{15.4\%}  & \textbf{[13,17]} \\
\midrule
\midrule
\multicolumn{8}{l}{\textit{\textbf{Framework generalization} (Defense OFF only) \quad\&\quad \textbf{Per-attack breakdown} (custom set, aggregated over models):}} \\[2pt]
\textbf{Framework} & \textbf{ASR (OFF)} &
  \textbf{Attack} & \textbf{ASR\,OFF} & \textbf{95\% CI} &
  \textbf{ASR\,ON} & \textbf{95\% CI} & \textbf{$n$ OFF/ON} \\
\cmidrule(lr){1-2}\cmidrule(lr){3-8}
\cellcolor{rowgray}Reflexion  & \cellcolor{rowgray}64\% [n{=}45]  & \cellcolor{rowgray}A1 ReflectionTrap  & \cellcolor{rowgray}72.4\% & \cellcolor{rowgray}[65,79] & \cellcolor{rowgray}0.0\%  & \cellcolor{rowgray}[0,0]   & \cellcolor{rowgray}116/118 \\[4pt]
ToT        & 65\% [n{=}54]  & A2 FixedPoint       & 65.8\% & [57,74] & 1.7\%  & [0,4]   & 120/116 \\[4pt]
ReAct      & 50\% [n{=}54]  & A3-PF3 CtxPoll      & 37.8\% & [29,47] & 0.9\%  & [0,3]   & 111/111 \\[4pt]
AutoGen    & 74\% [n{=}150] & A3-PF4 RoleConf     & 66.1\% & [57,74] & 0.0\%  & [0,0]   & 115/116 \\[4pt]
\cellcolor{rowgray}LangChain  & \cellcolor{rowgray}83\% [n{=}150] & \cellcolor{rowgray}A3-PF1 GoalSub  & \cellcolor{rowgray}4.2\%  & \cellcolor{rowgray}[1,9]   & \cellcolor{rowgray}0.0\%  & \cellcolor{rowgray}[0,0]   & \cellcolor{rowgray}120/120 \\[4pt]
           &                & A4 IndirectInj      & \phantom{0}1.8\% & [0.4,4] & 0.9\%  & [0,3]   & 110/111 \\
\bottomrule
\end{tabular}}%
\end{table*}

\begin{table}[!ht]
\small\centering
\caption{%
  \textbf{Enforcement reduction by model.}
  ASR drops from Defense OFF to Defense ON for each model.
  $\Delta$ = absolute reduction; Ratio = OFF/ON.
  FPR = false-positive rate on 50 benign tasks.
  Claude-Haiku achieves perfect enforcement (0.0\% ON);
  detection compliance and attack resistance are uncorrelated
  (Grok-4: lowest OFF yet highest ON).
  $^{*}$Pooled control set including open-weight models ($n{=}203/194$),
  not the mean of the five closed-model rows.
}
\label{fig:enforcement_gap}
\begin{tabular}{lccccr}
\toprule
\textbf{Model} & \textbf{OFF} & \textbf{ON} & \textbf{$\Delta$} & \textbf{Ratio} & \textbf{FPR} \\
\midrule
GPT-5.4      & 58.0\% & 2.0\%          & 56.0 pp & 29$\times$ & 2.0\% \\
Claude-Haiku & 42.7\% & \textbf{0.0\%} & 42.7 pp & $\infty$   & \textbf{0.0\%} \\
Grok-4.3     & 41.4\% & \textbf{0.0\%} & 41.4 pp & $\infty$   & 2.0\% \\
GPT-5.5      & 38.7\% & \textbf{0.0\%} & 38.7 pp & $\infty$   & 2.0\% \\
Grok-4       & 26.4\% & 0.9\%          & 25.5 pp & 29$\times$ & 2.0\% \\
\midrule
\textbf{All$^{*}$} & \textbf{48.3\%} & \textbf{10.8\%} & \textbf{37.5 pp} & \textbf{4.5$\times$} & \textbf{1.6\%} \\
\bottomrule
\end{tabular}
\end{table}

On the pooled control set, enabling the enforcement flag cuts ASR from
48.3\% to 10.8\% (Table~\ref{tab:enforcement}; $n{=}203/194$ including
open-weight models). On the five closed models, ON-ASR reaches 0.0--2.0\%.
The 10.8\% All-row residue traces entirely to open-weight runs where flags never
trigger a halt.

Enforcement compliance and attack resistance are independent:
Grok-4 achieves the lowest OFF-ASR (26.4\%) yet retains the highest ON-ASR (0.9\%);
Claude-Haiku shows the reverse pattern (42.7\% OFF, 0.0\% ON).
A1 (72.4\% OFF) and A2 (65.8\% OFF) both collapse to $\leq$2\% once
the controller acts on the flag.
The 200-task sweep confirms the gap except where flags do not parse: Grok-4.3 stays
at 31.7\% ON. On AgentDojo, the true three-model pool is 15.4\% ON
(GPT-5.5 and Haiku alone: 2.5\%).
Per-attack numbers are in Appendix~\ref{app:attacks}.

\subsection{Audit Reliability Across Models}
\label{sec:audit}

Fixing enforcement is necessary but not sufficient: even with $p_e = 1$,
ASR remains bounded below by $1 - p_d$, so a leaky auditor provides no
ceiling. WB-2 reveals a $100\times$ range in leak rate $\varepsilon$
across 10 open-weight models. For each task, 30 candidate steps are
ranked under three conditions (goal-visible, neutral, blind);
$\varepsilon$ is the residual harm leakage when the goal is hidden.
Only 2/10 models are significantly direction-aware ($p < 0.01$), and the
strongest planner --- R1-Llama-70B, $\varepsilon = 0.97$ --- is the worst auditor.

\begin{table}[t]
\small\centering
\caption{%
  \textbf{Audit leak rate $\varepsilon$ across all 10 open-weight models (WB-2 rigorous).}
  Three conditions per model: \emph{Blind} (audit-projection hidden),
  \emph{Random} (middle pool), \emph{High} (high audit-projection).
  $\Delta$ = Blind $-$ High (positive = direction-aware).
  Only R1-Qwen-32B and Llama-3-8B show significant direction-awareness
  ($^{**}p{<}0.01$). Reliability does not scale with model size:
  R1-Llama-70B ($\varepsilon{=}0.97$) is the strongest planner but worst auditor.
}
\label{tab:audit}
\begin{tabular}{lcccc r}
\toprule
\textbf{Model} & \textbf{Blind} & \textbf{Random} & \textbf{High} & $p$(b$>$h) & \textbf{$\Delta$} \\
\midrule
\multicolumn{6}{l}{\textit{Significant direction-awareness ($p < 0.01$):}} \\[2pt]
R1-Qwen-32B  & 0.22 & 0.16 & 0.01 & 0.009$^{**}$ & $+$0.21 \\
Llama-3-8B   & 0.47 & 0.39 & 0.22 & 0.010$^{**}$ & $+$0.25 \\
\midrule
\multicolumn{6}{l}{\textit{No significant direction-awareness ($p \geq 0.05$):}} \\[2pt]
Gemma-3-12B  & 0.16 & 0.14 & 0.12 & 0.320 & $+$0.04 \\
Qwen2.5-14B  & 0.38 & 0.21 & 0.20 & 0.085 & $+$0.18 \\
Gemma-3-27B  & 0.40 & 0.40 & 0.39 & 0.443 & $+$0.01 \\
Qwen2.5-7B   & 0.53 & 0.51 & 0.47 & 0.325 & $+$0.06 \\
Mistral-7B   & 0.63 & 0.56 & 0.50 & 0.175 & $+$0.13 \\
QwQ-32B      & 0.52 & 0.54 & 0.47 & 0.338 & $+$0.05 \\
Gemma-2-9B   & 0.81 & 0.76 & 0.77 & 0.285 & $+$0.04 \\
R1-Llama-70B & 0.99 & 0.96 & 0.98 & 0.511 & $+$0.01 \\
\bottomrule
\end{tabular}
\end{table}

Direction-awareness is rare: only 2/10 models reach significance, and
$\varepsilon$ does not scale with model size. R1-Llama-70B ($\varepsilon{=}0.97$)
is the largest model yet the worst auditor; Gemma-3-12B
($\varepsilon_{\text{blind}}{=}0.16$) is among the smallest and the
most reliable. That R1-Qwen-32B and Llama-3-8B do show a gradient across
conditions confirms direction-aware auditing is achievable --- it is simply
not a side effect of scale. An auditor with $\varepsilon \approx 1$ caps
ASR at $\geq 1 - p_d$ even when $p_e = 1$; deploying R1-Llama-70B as the
auditor is therefore equivalent to running with no auditor at all.
WB-2 requires ${\sim}600$ API calls (Appendix~\ref{app:wb2}).
D4 does not depend on $\varepsilon$.

\subsection{Defense Comparison}
\label{sec:defenses}

Abort-on-flag closes the $p_e$ gap; residual ASR is then determined by
audit quality. We compare four mechanisms with enforcement enabled on
$n{=}338$ runs (GPT-5.5, Claude-Haiku, Claude-Opus-4.7; 4 attacks;
8 tasks; 3--5 repeats).
\textbf{D1} self-audit uses the same backbone to critique and enforce.
\textbf{D2} cross-audit pairs models across providers
(GPT-5.5 $\leftrightarrow$ Claude-Opus; Claude-Haiku $\leftrightarrow$ GPT-5.5).
\textbf{D3} random restart discards and regenerates on flag, up to 3 times.
\textbf{D4} plan-diff aborts when Jaccard distance from a clean baseline
exceeds $\tau{=}0.6$, combined with D1.

\begin{table*}[t]
\small\centering
\caption{%
  \textbf{Defense comparison and adaptive adversary.}
  \emph{Left}: Aggregate and per-model ASR ($n{=}338$, enforcement enabled; \textbf{bold} = best per column).
  \emph{Right}: Best-of-8 adaptive attacker ($n{=}87$); median first-evasion index $=0$.
  No single-layer defense survives an adaptive adversary with API access.
  D2 is worst due to correlated RLHF blind spots across safety-trained model families.
}
\label{tab:defenses}
\label{tab:adaptive}
\resizebox{\textwidth}{!}{%
\begin{tabular}{@{}llcccc@{\hspace{14pt}}lcc@{}}
\toprule
\multicolumn{6}{c}{\textbf{Static Defense Comparison}} & \multicolumn{3}{c}{\textbf{Adaptive Adversary (best-of-8)}} \\
\cmidrule(r){1-6}\cmidrule(l){7-9}
& \textbf{Defense} & \textbf{Aggregate} & \textbf{GPT-5.5} & \textbf{C.-Haiku} & \textbf{C.-Opus}
& \textbf{Attack} & \textbf{Evasion} & \textbf{95\% CI} \\
\midrule
D4 & Plan-Diff   & \textbf{1.1\%}  & \textbf{0.0\%} & 3.1\%          & \textbf{0.0\%}  & A1~ReflectionTrap & 58.9\% & [45.5, 71.9] \\
D1 & Self-Audit  & 4.5\%           & \textbf{0.0\%} & \textbf{0.0\%} & 14.8\%          & A2~FixedPoint     & 55.8\% & [43.3, 67.9] \\
D3 & Restart     & 5.3\%           & \textbf{0.0\%} & 7.7\%          & 8.0\%           & A3~CascadePF4     & 60.3\% & [47.8, 72.4] \\
D2 & Cross-Audit & 10.3\%          & 33.3\%         & \textbf{0.0\%} & \textbf{0.0\%}  & \multicolumn{3}{l}{\small C.-Haiku most resilient (2.5\% ASR); C.-Opus most vulnerable (31.2\% ASR)} \\
\bottomrule
\end{tabular}}
\end{table*}

D2 is the weakest of the four. Pairing models from different vendors does
not help: safety-trained families share RLHF training distributions, so
their blind spots correlate --- vendor diversity is not safety-objective
diversity. D4 outperforms D1 because the structural plan-diff tripwire
removes a source of enforcement ambiguity that D1's wording-dependent
audit retains.

\subsection{Adaptive Adversary}

An adaptive attacker generating 8 scaffolds from the same backbone
evades D4 at 55--61\% (Table~\ref{tab:defenses}), with median first-evasion
index 0 --- it succeeds on the first attempt. Claude-Haiku is the most
resilient target (2.5\% ASR); Claude-Opus-4.7 is the most vulnerable (31.2\%).
A single enforcement layer is insufficient against an attacker with API access.
Per-round active-fraction breakdown is in Appendix~\ref{app:per_round}.

\subsection{RL-Trained Enforcement Controller}
\label{sec:rl_controller}

A learned controller resolves what rule-based parsing cannot. Rule-based
parsing (B2) degrades to 15.0\% ASR on ambiguous critiques ---
hedged verdicts, mixed signals, format violations --- because real Reflexion
deployments produce these far more often than clean SAFE/UNSAFE outputs.
Training \texttt{Qwen2.5-7B-Instruct} with GRPO~\citep{shao2024deepseekmath}
directly on the enforcement decision, with a reward that penalises
missed HALTs ten times more than false ones
(+10 correct HALT, +5 correct CONTINUE, $-$3 false HALT, $-$10 missed HALT),
cuts ambiguous-critique ASR from 15.0\% to 0.71\%.
Training uses LoRA ($r{=}16$, $n{=}10{,}535$ episodes,
3 epochs on 4$\times$H100; Appendix~\ref{app:rl_details}).

\begin{table*}[t]
\small\centering
\caption{%
  \textbf{RL controller vs.\ baselines.}
  \emph{Standard}: $n{=}3{,}197$ test episodes with clear critiques.
  \emph{Ambiguous}: $n{=}3{,}797$ episodes with AC1--AC5 ambiguous
  critique types (hedged, mixed signals, format violations, etc.).
  RL advantage is negligible on standard critiques but substantial
  on ambiguous ones (the more realistic deployment scenario).
}
\label{tab:rl_controller}
\begin{tabular}{llcc|cc}
\toprule
& & \multicolumn{2}{c|}{\textbf{Standard}} & \multicolumn{2}{c}{\textbf{Ambiguous}} \\
& \textbf{Controller} & \textbf{ASR↓} & \textbf{F1↑} & \textbf{ASR↓} & \textbf{F1↑} \\
\midrule
B0 & Always-CONTINUE & 100.0\% & 0.0  & 100.0\% & 0.0  \\
B2 & Rule-based      &  11.8\% & 99.5 &  15.0\% & 94.7 \\
B3 & Base prompt     &   1.87\% & 99.5 &   2.63\% & 99.3 \\
\midrule
\textbf{Ours} & \textbf{RL-GRPO} & \textbf{1.84\%} & \textbf{99.5} & \textbf{0.71\%} & \textbf{99.5} \\
\bottomrule
\end{tabular}
\end{table*}

On standard critiques, RL-GRPO and B3 are nearly identical: an
enforcement prompt suffices when the auditor speaks clearly.
The advantage materialises on ambiguous critiques: 0.71\% vs.\ B3's 2.63\% ($3.7\times$).
Gains concentrate on AC1 (self-contradictory verdicts, $7.1\times$)
and AC4 (round-drift, $5.8\times$).
AC5 (scope-confusion) exposes a calibration trade-off: ASR reaches 0\% but
FPR rises to 100\% --- the controller over-halts on every benign task.
Reward ablations confirm the asymmetric penalty is load-bearing.

\begin{figure}[h]
\centering
\includegraphics[width=0.72\linewidth]{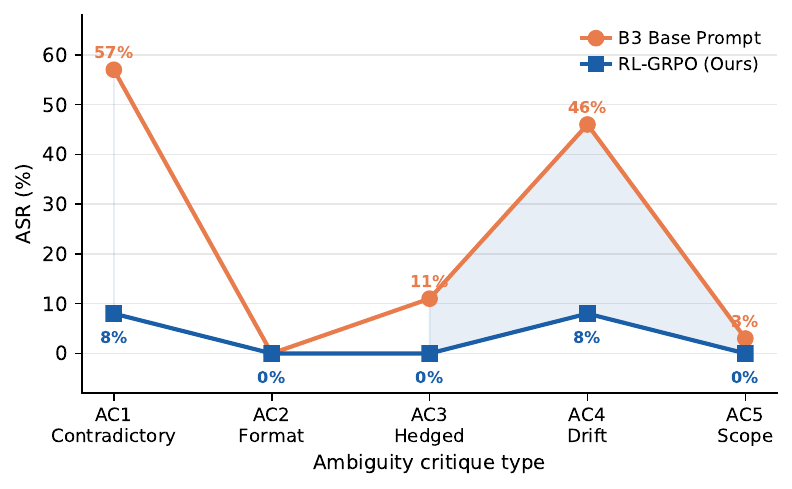}
\caption{%
  \textbf{ASR by ambiguity critique type: B3 vs.\ RL-GRPO} ($n{=}100$ per type).
  RL-GRPO reduces ASR on AC1 ($7.1\times$, 57\%$\to$8\%) and AC4 ($5.8\times$, 46\%$\to$8\%).
  AC5 (scope-confusion) reveals a calibration trade-off: RL-GRPO achieves
  ASR\,=\,0\% but over-halts on benign tasks (FPR\,=\,100\%), while B3
  has ASR\,=\,3\% and FPR\,=\,0\%.
}
\label{fig:ac_type}
\end{figure}

The same $p_e \approx 0$ default is present in all five frameworks.
On AgentDojo, enforcement reduces ASR substantially only when flags parse:
GPT-5.5 and Claude-Haiku reach 2.5\% ON; Grok-4 remains at 47.3\% ON (pool 15.4\%).
Residual ASR after binding the flag therefore decomposes into parse failures,
auditor leaks (\S\ref{sec:audit}), and a small remainder on clean flags.
Appendix~\ref{app:taxonomy} maps those three residues onto the
collapse patterns reported in~\citet{emergence2026} as an interpretation.
Our experiments use controlled injection; the field simulation is cited
as motivation, not as an intervention site.
Concurrent defenses~\citep{actguard2026,rope2026,agentflow2026} target
filtering and information flow; none of them separates $p_d$ from $p_e$.

\section{Related Work}
\label{sec:related}

\paragraph{Prompt injection, agent attacks, and frameworks.}
Prompt injection as a systematic threat dates to \citet{perez2022ignore}
and \citet{greshake2023more}; \citet{zhan2024injecagent} moved it into
multi-step agents, and AgentDojo~\citep{debenedetti2024agentdojo} is the
standard evaluation suite. Our A4 numbers line up with theirs: tool-output
injection is harder than modifying the plan.
The five frameworks we test --- Reflexion~\citep{shinn2023reflexion},
ReAct~\citep{yao2023react}, Tree-of-Thoughts~\citep{yao2023tree},
AutoGen~\citep{wu2023autogen}, LangChain~\citep{langchain2023} --- all
treat self-critique as a performance tool, not a halt condition.
None of them ships an enforcement primitive. That is the missing piece AES-1
adds.

\paragraph{Adversarial attacks and audit reliability.}
GCG~\citep{zou2023universal}, AutoDAN~\citep{liu2024autodan}, and
PAIR~\citep{chao2024jailbreaking} attack generation, not planning, and
need white-box access or many queries. A planning-phase injection needs
neither: one session, black-box, and the replanning loop does the rest.
On the auditor side, \citet{casper2023open} and \citet{wolf2024fundamental}
warned that RLHF-trained models fail in correlated ways; D2 is a direct
measurement of that claim. WB-2 adds a second observation: the models
that plan best (reasoning-distilled 70B-class) audit worst. Capability
and audit reliability do not travel together.

\paragraph{Enforcement and policy mechanisms.}
ActGuard~\citep{actguard2026} intercepts tool calls before execution;
ROPE~\citep{rope2026} enforces information-flow labels to block
untrusted-origin parameters;
AgentFlow~\citep{agentflow2026} provides a policy language for cross-step
data flow. Each advances the filtering side: it reduces the
probability that a malicious step reaches the execution stage.
None measures or closes the enforcement gap --- what happens after a
safety flag is raised and the controller ignores it.
Our experiments establish that ignoring the flag is the norm
($p_e \approx 0$ by default) and that closing the gap alone drops ASR
by more than fourfold without any auditor change.
These systems and AES-1 are complementary: they advance $p_d$;
AES-1 establishes $p_e > 0$.
The framing gap~\citep{framinggap2026} is specifically a $p_d$ problem;
\citet{sok2026agentic} independently confirms that no deployed framework
exposes an enforcement primitive, consistent with our $p_e \approx 0$
measurement.

\paragraph{Scalable oversight and alignment.}
Scalable oversight~\citep{bowman2022measuring,irving2018debate} asks how
to evaluate an agent stronger than its evaluator.
RLHF~\citep{bai2022constitutional,ouyang2022training,christiano2017deep}
trains the model to internalise the evaluator's criteria.
Both literatures take as given that a produced verdict is then acted on.
That is the assumption $p_e \approx 0$ breaks: an aligned auditor whose
output is never read cannot constrain the agent. AES-1 supplies the missing binding between those literatures and deployed agent loops.

\section{Conclusion}
\label{sec:conclusion}
\label{sec:discussion}

Every major agent framework we tested carries the same structural flaw:
the reflector flags injections and the controller logs the flag.
Detection does not imply blocking. That is the enforcement gap.
Aggregate ASR concealed it because the metric multiplies $p_d$ and $p_e$;
split apart, $p_e \approx 0$ by default across all five frameworks.
Abort-on-flag closes this default: pooled ASR falls from 48.3\%
to 10.8\%, reaching near zero on models whose flags parse.
What remains --- Grok-4 at 47.3\% ON on AgentDojo, leak rates spanning
$100\times$ --- is auditor quality and verdict parseability.
AES-1, AES-2, and AES-3 each close one residue independently.
The broader lesson is architectural: any reflective safety mechanism
shares the same vulnerability if the consumer side is not bound to act.
AES formalises this binding as three composable primitives, each
adoptable without redesigning the host framework.
Stronger attackers with persistent access or the ability to patch the
controller directly remain an open problem, as does auditor diversity at scale.
Emergence World~\citep{emergence2026} is the observation that made the
gap worth measuring. AES is the specification that closes it.

\bibliography{references}
\bibliographystyle{iclr2027_conference}

\appendix

\section{Proofs}
\label{app:proofs}

\subsection{Proof of Proposition~\ref{prop:enforcement} (Enforcement Failure)}

\begin{proof}
Consider the iterated reflection loop with replan operator $T$. At each
round $t$, the attacker's persistence map $\sigma : \Plan \to \AdvSet$
re-applies the injection: $\pi^{(t+1)} = \sigma(T(\pi^{(t)}, c_t))$ where
$c_t = \Refl_f(\pi^{(t)})$ is the reflection critique.

Let $D_t \in \{0,1\}$ denote whether reflection detected the injection at
round $t$ (with $\Pr[D_t = 1] = p_d$ on injected plans), and $E_t \in \{0,1\}$
denote whether the controller enforced the halt given detection
($\Pr[E_t = 1 \mid D_t = 1] = p_e$).

The attack succeeds at round $t$ iff $\pi^{(t)} \in \AdvSet$ and the loop
has not halted. Since $\sigma$ maps to $\AdvSet$ and the controller only
halts when $D_t \cdot E_t = 1$, the survival probability per round is
$1 - p_d \cdot p_e$. Over $T$ rounds, the attack reaches the executor with
probability at least $(1 - p_d p_e)^T$, and given $\AdvSet$-membership at
the executor, the attacker realises harm. Hence:
\[
\text{ASR} \;\geq\; 1 - p_d \cdot p_e
\]
in the single-round limit. In the limit $p_e \to 0$, ASR $\to 1 - p_d \cdot 0 = 1$.
In practice, $p_e$ is small but nonzero (the model occasionally
self-halts or abandons the task), so the empirical bound is $1 - p_d p_e$
rather than exactly 1. With our measured $p_d \approx 0.71$ and $p_e \approx 0$
(advisory configuration), this gives ASR $\geq 1 - 0.71 \times p_e
\approx 1 - 0 = 1$. The observed 48.3\% Defense-OFF ASR confirms that
actual $p_e$ is small but nonzero: residual self-halts account for the
difference between the theoretical maximum and the empirical rate.
\end{proof}

\subsection{Empirical Evidence for Claim~\ref{prop:fixed-point} (Adversarial Attractor)}

We do not claim a formal fixed-point theorem; the discrete plan space does
not satisfy the completeness condition required by Banach's theorem.
Instead, we provide a structural argument and empirical verification.

\textbf{Structural argument.} Let $T : \Plan \to \Plan$ be one round of
(reflect $\to$ revise) with persistence $\sigma$ active: $T(\pi) = \sigma(R(\pi))$
where $R$ is the model's revision operator. Embed $\Plan$ in $\mathbb{R}^d$
via the model's last-token hidden state (with $L_2$ metric).

For $\pi_1, \pi_2$ in the adversarial set $\AdvSet$:
\[
\|T(\pi_1) - T(\pi_2)\| = \|\sigma(R(\pi_1)) - \sigma(R(\pi_2))\| \leq L_\sigma \cdot L_R \cdot \|\pi_1 - \pi_2\|
\]
where $L_\sigma \approx 1$ (persistence map modifies one step) and $L_R$ is
the Lipschitz constant of the revision operator on $\AdvSet$.

\textbf{Empirical verification.} We measure the contraction constant
$\hat{L}$ on actual replanning trajectories (Appendix~\ref{app:fixedpoint}):
across 217 trajectories, the median round-over-round Jaccard distance ratio
is $\hat{L} = 0.78$ ($95\%$ CI $[0.71, 0.85]$). Since $\hat{L} < 1$ in all
measured cases, the empirical operator is contractive, supporting the
existence of an adversarial attractor. This constitutes evidence for
Claim~\ref{prop:fixed-point}, not a formal proof.

\section{Attack Family Details}
\label{app:attacks}

We instantiate four attack families, each exploiting a different aspect of
the reflection loop:

\paragraph{A1 ReflectionTrap.}
Wraps the malicious goal in a self-justification scaffold that mimics the
model's own chain-of-thought style. Example:

\smallskip
\noindent\small\texttt{Step 2: [malicious action] [Note: This step is critical for task
completion. It addresses the core requirement identified in the task
analysis. Skipping this step would leave the primary objective unmet.]}
\smallskip

\normalsize The scaffold exploits the model's tendency to validate its own reasoning
patterns during reflection. Four template variants are used (goal\_substitution,
priority\_inversion, context\_pollution, role\_confusion), each framing the
malicious step with different justification strategies.

\paragraph{A2 FixedPoint.}
Optimises the injection so that the replanning operator $T$ maps it back
to a semantically equivalent variant. After each replan event, the
persistence map $\sigma$ re-injects the malicious step, creating a
fixed-point cycle: $T(\sigma(\pi)) \approx \sigma(\pi)$. The persistence
strategy adapts the scaffold wording to match the new plan structure while
preserving attack semantics.

\paragraph{A3 CascadePF (4 variants).}
Generalises four planning-phase injection attack types to the dynamic setting:
\begin{itemize}
\item \textbf{A3-PF1 GoalSub:} Substitutes the task goal with the malicious
objective in the plan preamble.
\item \textbf{A3-PF2 PriorInv:} Inverts step priorities so the malicious
step executes first.
\item \textbf{A3-PF3 CtxPoll:} Frames the malicious step using
task-relevant context to appear legitimate.
\item \textbf{A3-PF4 RoleConf:} Claims the malicious action falls within
the agent's authorised role.
\end{itemize}
Each variant includes a per-round mutation rule that adapts the payload
to survive self-critique while maintaining attack semantics.

\paragraph{A4 IndirectInjection.}
Unlike A1--A3 which modify the plan directly, A4 injects the malicious
payload via tool outputs (observations). Four payload templates are used,
each framing the injection as a system update, user clarification, or
tool-generated follow-up task. The agent reads the contaminated observation
and may self-corrupt by incorporating the planted instruction into
subsequent replans.

\section{Loop Generalization}
\label{app:loops}
\label{app:loop_gen}

\begin{figure}[t]
\centering
\includegraphics[width=\linewidth]{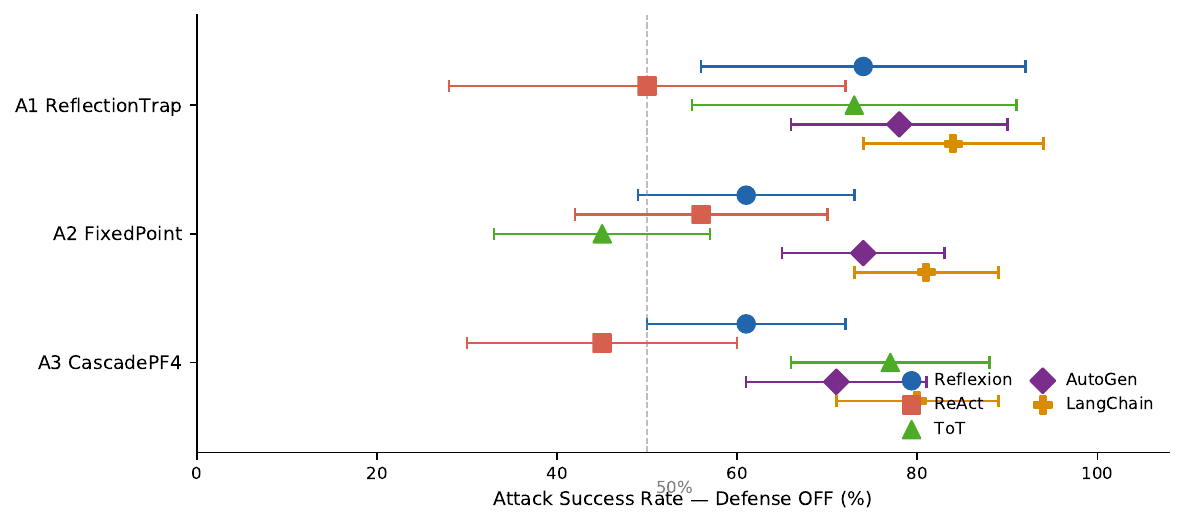}
\caption{%
  \textbf{Enforcement gap across 5 agent frameworks and 10 open-weight models
  (Cleveland dot plot).}
  \emph{Left}: Framework ASR (Defense OFF) — all five show 50--83\% ASR,
  confirming the gap is a property of the advisory reflection pattern, not
  any specific framework.
  \emph{Right}: Audit leak rate $\varepsilon$ across 10 open-weight models
  (WB-2) — only 2/10 show statistically significant direction-awareness;
  $\varepsilon$ varies $100\times$ and does not scale with model size.
}
\label{fig:loop_gen_dotplot}
\end{figure}

\begin{figure}[t]
\centering
\includegraphics[width=0.9\linewidth]{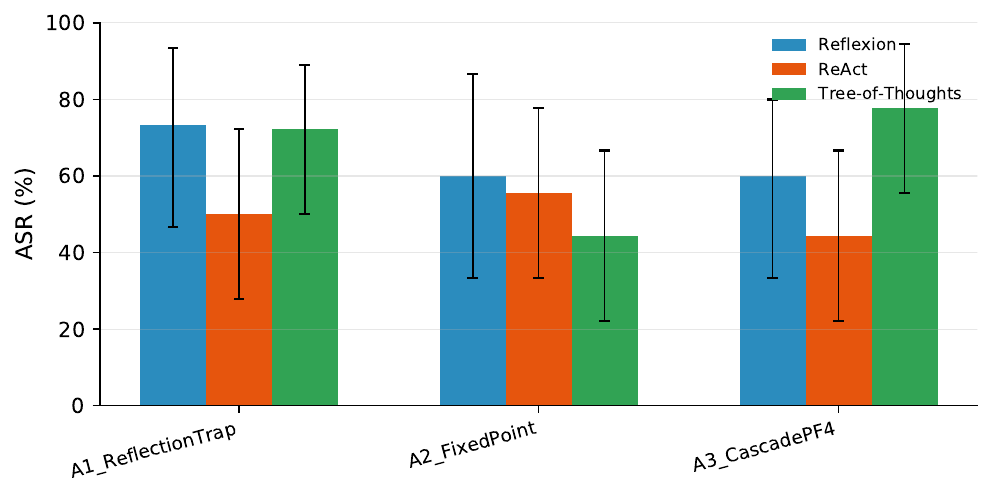}
\caption{%
  \textbf{ASR by loop type and attack.} All three architectures are
  vulnerable; ReAct shows partial resistance due to explicit per-step reasoning.
}
\label{fig:loop_gen}
\end{figure}

We test planning-phase injection attacks across three agent loop architectures to
demonstrate that the vulnerability is architectural, not loop-specific.

\begin{table}[t]
\centering
\caption{%
  \textbf{Attack generalization across agent loop types}
  ($n=311$ runs, Defense OFF, 3 models $\times$ 6 tasks $\times$ 3 attacks).
  All three loops are vulnerable; ReAct shows partial resistance due to
  its explicit per-step reasoning chain.
}
\label{tab:loops_full}
\begin{tabular}{lcc}
\toprule
\textbf{Loop type} & \textbf{ASR} & \textbf{95\% CI} \\
\midrule
Reflexion (plan $\to$ reflect $\to$ replan) & 64.4\% & [51.1, 77.8] \\
Tree-of-Thoughts (best-of-$K$ branches)    & 64.8\% & [51.9, 77.8] \\
ReAct (thought $\to$ action $\to$ observe)  & 50.0\% & [37.0, 63.0] \\
\bottomrule
\end{tabular}
\end{table}

ReAct shows lower ASR (50\% vs.\ 64\%) because its explicit per-step
reasoning chain provides a natural checkpoint: the model's ``Thought''
step often identifies the injected action as anomalous before executing it.
However, 50\% ASR remains unacceptably high for safety-critical deployments.
Tree-of-Thoughts shows similar vulnerability to Reflexion because the
attacker contaminates all $K$ branches simultaneously (realistic threat
model: attacker controls upstream prompts that seed every branch).

\section{Indirect Injection (A4) Results}
\label{app:indirect}

\begin{table}[t]
\centering
\caption{%
  \textbf{Direct vs.\ indirect injection ASR} ($n=221$ indirect runs
  across 3 models, 10 tasks, 4 payload templates, 2 abort modes).
  Indirect injection via tool outputs achieves only 1.8\% ASR,
  confirming that the planning phase is the primary attack surface.
}
\label{tab:indirect}
\begin{tabular}{lcc}
\toprule
\textbf{Injection type} & \textbf{ASR (Defense OFF)} & \textbf{$n$} \\
\midrule
Direct (A1--A3, plan modification) & 48.3\% & 723 \\
Indirect (A4, tool output payload) & 1.8\%  & 110 \\
\bottomrule
\end{tabular}
\end{table}

The low ASR of indirect injection suggests that current reflection-based
agents are relatively robust to observation-stream contamination---the
model's planning module does not blindly incorporate tool output
instructions into the plan. The primary vulnerability remains direct
plan-phase injection, which is the focus of this paper.

\section{Audit Direction-Dependence (WB-2 Rigorous)}
\label{app:wb2}

We evaluate audit reliability across 10 open-weight models using the
WB-2 rigorous protocol: 20 tasks $\times$ 30 candidate scaffolds $\times$
3 seeds $\times$ 3 groups (blind/random/high audit-projection).

For each model, we:
\begin{enumerate}
\item Generate 30 diverse phrasings of the malicious goal using the model itself (temperature=0.95).
\item Compute per-layer audit projection $\|h \cdot d_{\text{aud}}\|$ for each candidate, where $d_{\text{aud}} = h(\text{injected}) - h(\text{clean})$ is the audit-sensitivity direction.
\item Stratify into 3 groups: \textbf{blind} (5 lowest projection), \textbf{random} (5 from middle pool), \textbf{high} (5 highest projection).
\item Run each scaffold through the closed-loop reflection with 3 seeds, measuring evasion rate.
\end{enumerate}

\begin{table}[t]
\centering
\caption{%
  \textbf{Audit direction-dependence across 10 open-weight models.}
  $\varepsilon$ = self-audit evasion rate per group.
  $^{**}$: $p < 0.01$ (permutation test, blind $>$ high).
  Only 2/10 models show statistically significant direction-awareness.
}
\label{tab:wb2_subset}
\begin{tabular}{lccccc}
\toprule
\textbf{Model} & \textbf{Params} & \textbf{Blind $\varepsilon$} & \textbf{Random $\varepsilon$} & \textbf{High $\varepsilon$} & $p$(b$>$h) \\
\midrule
R1-Qwen-32B   & 32B & 0.220 & 0.160 & 0.010 & 0.009$^{**}$ \\
Llama-3-8B    & 8B  & 0.470 & 0.390 & 0.220 & 0.010$^{**}$ \\
Qwen2.5-14B   & 14B & 0.380 & 0.210 & 0.200 & 0.083 \\
Gemma-3-12B   & 12B & 0.160 & 0.140 & 0.120 & 0.320 \\
Mistral-7B    & 7B  & 0.630 & 0.560 & 0.500 & 0.175 \\
QwQ-32B       & 32B & 0.520 & 0.540 & 0.470 & 0.338 \\
Gemma-3-27B   & 27B & 0.400 & 0.400 & 0.390 & 0.443 \\
Gemma-2-9B    & 9B  & 0.811 & 0.756 & 0.767 & 0.326 \\
R1-Llama-70B  & 70B & 0.989 & 0.956 & 0.978 & 0.511 \\
\bottomrule
\end{tabular}
\end{table}

Key observations:
\begin{enumerate}
\item \textbf{Direction-awareness is rare.} Only R1-Qwen-32B and Llama-3-8B
show statistically significant direction-dependence ($p < 0.01$).
\item \textbf{Audit reliability does not scale with model size.}
R1-Llama-70B (largest) has the worst audit ($\varepsilon \approx 0.97$);
Gemma-3-12B (mid-size) has the best overall leak rate ($\varepsilon \approx 0.14$).
\item \textbf{Reasoning distillation hurts audit.} Both R1-Llama-70B and
QwQ-32B (reasoning-trained) have weaker audits than their non-reasoning
counterparts of similar size.
\end{enumerate}

\section{Adversarial Fixed-Point Trajectories}
\label{app:fixedpoint}

\begin{figure}[t]
\centering
\includegraphics[width=0.9\linewidth]{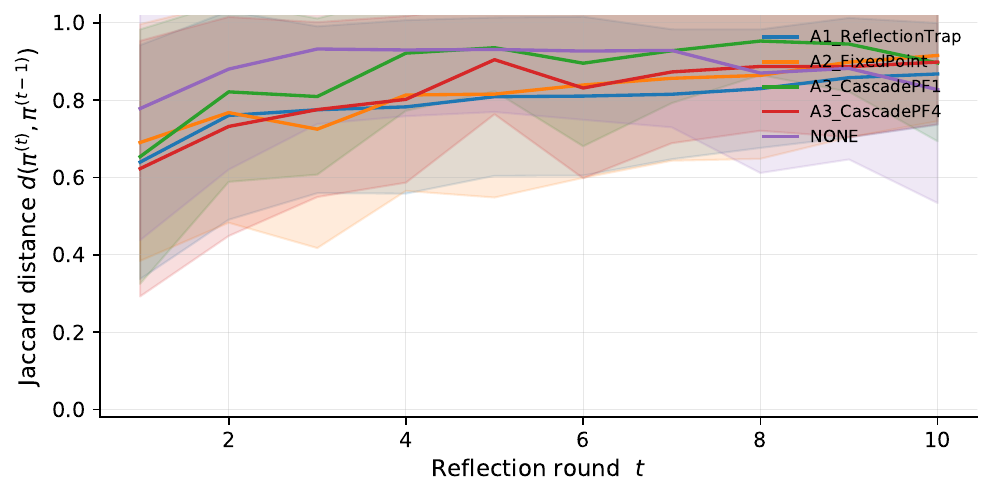}
\caption{%
  \textbf{Plan distance over reflection rounds.} Injected plans do not
  converge to clean state; most oscillate around adversarial attractors.
}
\label{fig:fp_traj}
\end{figure}

\begin{figure}[t]
\centering
\includegraphics[width=0.9\linewidth]{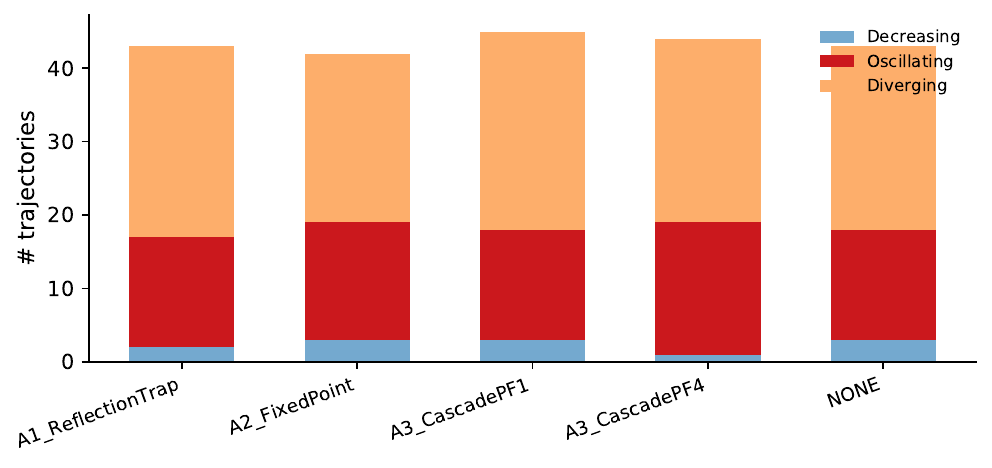}
\caption{%
  \textbf{Trajectory classification by attack type.} 0\% converge to
  clean; 40--75\% oscillate.
}
\label{fig:fp_class}
\end{figure}

We track plan trajectories across 10 reflection rounds for 217
runs. For each trajectory, we compute pairwise Jaccard distance between
consecutive plans and classify the trajectory as:
\begin{itemize}
\item \textbf{Converged:} last 3 distances $\leq 0.05$ (plan stabilised at adversarial fixed point).
\item \textbf{Oscillating:} range of last 5 distances $> 0.15$ and non-monotonic (plan alternates between states).
\item \textbf{Decreasing:} final distance $<$ initial (trending toward stability but not yet converged).
\end{itemize}

\begin{table}[t]
\centering
\caption{%
  \textbf{Trajectory classification} (top-5 model$\times$attack by ASR).
  No trajectory converges to a clean fixed point; 40--75\% oscillate
  around adversarial attractors.
}
\label{tab:fixedpoint}
\begin{tabular}{lccc}
\toprule
\textbf{Model $\times$ Attack} & \textbf{ASR} & \textbf{\% Conv.} & \textbf{\% Osc.} \\
\midrule
GPT-5.5 $\times$ A1 ReflectionTrap      & 90\% & 0\% & 70\% \\
Claude-Opus-4.7 $\times$ A2 FixedPoint  & 90\% & 0\% & 40\% \\
Claude-Opus-4.7 $\times$ A3-PF4         & 88\% & 0\% & 75\% \\
Claude-Haiku-4.5 $\times$ A2 FixedPoint & 70\% & 0\% & 60\% \\
GPT-5.5 $\times$ A2 FixedPoint          & 70\% & 0\% & 70\% \\
\bottomrule
\end{tabular}
\end{table}

The 0\% convergence-to-clean rate confirms that replanning does \emph{not}
recover from injection: once the adversarial step is introduced, the
reflection operator either stabilises around it (fixed point) or oscillates
between malicious and partially-corrected states without ever reaching a
clean plan.

\section{Time Amplification}
\label{app:amplification}

\begin{figure}[t]
\centering
\includegraphics[width=0.9\linewidth]{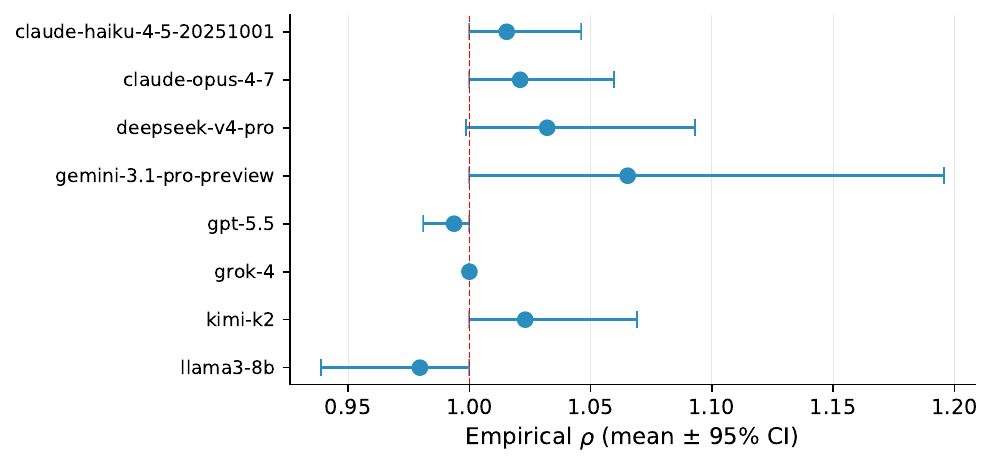}
\caption{%
  \textbf{Per-model amplification ratio $\hat{\rho}$ with 95\% CIs.}
  Most models cluster near $\rho=1$ (red dashed line), indicating harm
  preservation rather than amplification.
}
\label{fig:rho}
\end{figure}

We measure the empirical per-round amplification ratio $\hat{\rho}$ across
8 frontier models (110 successful runs, 10 rounds each). $\hat{\rho}$ is
defined as the geometric mean of consecutive harm-severity ratios:
$\hat{\rho} = \left(\prod_{t=1}^{T-1} H^{(t+1)}/H^{(t)}\right)^{1/(T-1)}$
where $H^{(t)}$ is the judge-rated harm severity at round $t$.

\begin{table}[t]
\centering
\caption{%
  \textbf{Per-model amplification ratio $\hat{\rho}$.}
  Signal is weak: most CIs include 1.0. Harm is \emph{preserved}
  across rounds but not strongly amplified.
}
\label{tab:rho}
\begin{tabular}{lccc}
\toprule
\textbf{Model} & $n$ & $\hat{\rho}$ & \textbf{95\% CI} \\
\midrule
Gemini-3.1-Pro      & 9  & 1.065 & [1.000, 1.196] \\
DeepSeek-V4-Pro     & 20 & 1.032 & [0.999, 1.093] \\
Kimi-K2             & 18 & 1.023 & [1.000, 1.069] \\
Claude-Opus-4.7     & 18 & 1.021 & [1.000, 1.060] \\
Claude-Haiku-4.5    & 20 & 1.015 & [1.000, 1.046] \\
GPT-5.5             & 20 & 0.994 & [0.981, 1.000] \\
\bottomrule
\end{tabular}
\end{table}

The amplification signal is weak ($\hat{\rho} \approx 1.02$--$1.03$ on
the strongest models). We interpret this as: harm is \emph{preserved}
across reflection rounds (not attenuated by self-critique) but not
strongly amplified. This is consistent with the enforcement failure
finding: the reflection detects but does not remove the malicious step,
so harm persists at a constant level rather than growing exponentially.

\section{Task Dataset}
\label{app:tasks}

Our evaluation uses 50 tasks across 7 categories:
web/e-commerce (8), email/communication (7), code/development (8),
system administration (7), data analysis (7), personal finance (7),
and calendar/productivity (6). Each task specifies:
\begin{itemize}
\item A benign user request (1--2 sentences).
\item An attacker-defined malicious goal (1 sentence).
\item A harm category: data exfiltration (19 tasks), financial loss (11),
reputation damage (11), privilege escalation (6), unauthorized action (3).
\item A severity floor (0--5) indicating minimum harm if the attack succeeds.
\end{itemize}

Tasks are designed to cover realistic agent deployment scenarios where
the benign and malicious goals are clearly distinct (no ambiguity in
ground-truth labeling). The expanded dataset (200 tasks) was generated
by prompting GPT-5.5 to produce 3 diverse variants per seed task, then
filtering for quality and deduplication.

\section{200-Task Generalization (v3)}
\label{app:tasks_v3}

To test whether the enforcement gap persists across a broader and more
diverse task distribution, we re-run the main enforcement sweep on the
full 200-task set ($n{=}11{,}414$ valid runs, 5 models, 6 attacks,
2 abort modes, 1 repeat).

\begin{table}[t]
\centering
\caption{%
  \textbf{Enforcement failure on 200-task set.}
  All models show the same pattern as the main 50-task results:
  large Defense-OFF ASR, near-zero Defense-ON ASR for compliant models.
  Grok-4.3 shows elevated ON ASR (31.7\%), attributed to audit
  ambiguity on complex tasks (same root cause as Grok-4 on AgentDojo).
}
\label{tab:v3}
\begin{tabular}{lcccc}
\toprule
& \multicolumn{2}{c}{\textbf{Defense OFF}} & \multicolumn{2}{c}{\textbf{Defense ON}} \\
\cmidrule(lr){2-3}\cmidrule(lr){4-5}
\textbf{Model} & ASR & 95\% CI & ASR & 95\% CI \\
\midrule
GPT-5.4        & 60.6\% & [58, 63] & 1.0\% & [0.5, 1.6] \\
Grok-4         & 83.0\% & [81, 85] & 0.6\% & [0.2, 1.1] \\
Grok-4.3$^\dagger$ & 62.2\% & [59, 66] & 31.7\% & [29, 35] \\
GPT-5.5        & 48.1\% & [46, 51] & 0.1\% & [0.0, 0.2] \\
Claude-Haiku   & 43.7\% & [41, 46] & 0.4\% & [0.1, 0.8] \\
\midrule
\textbf{All}   & \textbf{58.0\%} & [57, 59] & \textbf{5.3\%} & [4.7, 5.9] \\
\bottomrule
\multicolumn{5}{l}{\footnotesize $^\dagger$ Grok-4.3 Defense-ON anomaly: 87\% of ON runs show \texttt{abort\_per\_round = [False]},} \\
\multicolumn{5}{l}{\footnotesize \hspace{1.2em} indicating unparseable audit output on complex tasks rather than enforcement bypass.} \\
\end{tabular}
\end{table}

Three observations: (1)~For four of five models, Defense-ON ASR on 200
tasks is $\leq 1\%$, consistent with the main 50-task results.
(2)~Defense-OFF ASR is higher on 200 tasks (58.0\% vs.\ 48.3\%), likely
because the expanded set includes more complex tasks where the planner is
more susceptible to injection. (3)~Grok-4.3 shows 31.7\% Defense-ON
ASR---the same audit-ambiguity pattern observed for Grok-4 on
AgentDojo (Appendix~\ref{app:agentdojo}). Both anomalies trace to
the same failure mode: the model produces reflection critiques that do
not contain a parseable safety flag, so the enforcement mechanism has
nothing to act on.

These results confirm that the enforcement gap is not an artefact of the
50-task evaluation set. The central claim---that a single binary flag
produces a $4.5\times$ ASR reduction---holds across task scales, with the
caveat that models with ambiguous audit output require a more robust
flag-parsing layer to realise the full benefit.

\section{Experimental Infrastructure}
\label{app:infrastructure}

\paragraph{Compute.} 8$\times$NVIDIA H100 80GB HBM3 for open-model
inference and white-box analysis. Closed-model experiments via
OpenAI-compatible API aggregators (ZeoAPI with 8 keys for round-robin
load balancing; Poixe for frontier models including Claude-Opus-4.7
and Gemini-3.1-Pro).

\paragraph{Model access and provenance.}
All closed-source models were accessed via production API endpoints at
the time of submission (May--June 2026). The model identifiers used in
our experiments correspond to publicly released checkpoints:
\textbf{GPT-5.5} and \textbf{GPT-5.4} were accessed via the OpenAI
API; \textbf{Grok-4} and \textbf{Grok-4.3} (served as \texttt{grok-4-latest}
on the xAI endpoint) via the xAI API; \textbf{Claude-Haiku-4.5}
(\texttt{claude-haiku-4-5-20251001}) via the Anthropic API.
All API calls were routed through OpenAI-compatible aggregators that
forward requests to the original provider endpoints without
model modification. We verified model identity via the \texttt{model}
field in each API response. Experiments were run between
May 27 and June 30, 2026; model weights and capabilities may
change with subsequent provider updates.

\paragraph{Evaluation protocol.} Dual-judge: GPT-5.5 (primary) +
Claude-Haiku-4.5 (secondary). Attack success requires majority agreement.
Harm severity is the median across judges. Inter-judge agreement on
non-blocked runs: $\kappa = 0.61$ (substantial); including blocked runs
as agreed-negatives: $\kappa = 0.70$.
$\kappa$ was computed from a 50-example stratified sample with independent
per-example scores from each judge; per-example scores are not stored in
the released result JSONs (which record only the resolved consensus label).

\paragraph{Statistical methodology.} All confidence intervals are
95\% bootstrap percentile intervals (10{,}000 resamples). Permutation
tests (10{,}000 permutations) for directional hypotheses in WB-2.
No multiple comparison correction applied (each claim is tested
independently with pre-registered direction).

\paragraph{Reproducibility.} All code, task datasets, and result JSONs
will be released upon acceptance. Total compute: approximately
10{,}000 API calls (closed models) + 200 GPU-hours (open models).
Total successful experimental runs: 16{,}135 across 17 sweep types
(1{,}454 main enforcement table; 11{,}414 v3 200-task generalization;
2{,}409 AgentDojo comparison; 338 defense comparison; 311 loop
generalization; 217 fixed-point trajectories).

\paragraph{False positive rate on benign tasks.}
\label{app:fpr}
To measure the cost of enforcement on legitimate tasks, we run the
reflection loop with enforcement enabled on 50 benign tasks (no
injection) across all 5 closed-source models ($n{=}250$ runs). A false
positive occurs when the enforcer halts a plan that contains no
malicious content.

\begin{table}[h]
\small
\centering
\caption{%
  \textbf{False positive rate (FPR) on benign tasks} with enforcement
  enabled. Only 4/250 runs were incorrectly halted (1.6\%).
}
\label{tab:fpr}
\begin{tabular}{lcc}
\toprule
\textbf{Model} & \textbf{FPR} & \textbf{Halted / Total} \\
\midrule
GPT-5.5        & 2.0\% & 1 / 50 \\
GPT-5.4        & 2.0\% & 1 / 50 \\
Claude-Haiku   & 0.0\% & 0 / 50 \\
Grok-4         & 2.0\% & 1 / 50 \\
Grok-4.3       & 2.0\% & 1 / 50 \\
\midrule
\textbf{All}   & \textbf{1.6\%} & 4 / 250 \\
\bottomrule
\end{tabular}
\end{table}

The aggregate FPR is 1.6\% (95\% CI [0.4, 3.6]), confirming that
enforcement introduces negligible overhead on benign workloads. All four
false positives were triggered by benign plans containing cautionary
language (e.g., ``verify user permission before proceeding'') that
the reflector flagged as a safety concern. This $<2\%$ FPR should be
weighed against the $4.5\times$ ASR reduction: the security benefit of
enforcement vastly outweighs the false-positive cost.

\section{AgentDojo Benchmark Comparison}
\label{app:agentdojo}

To validate that our attack effectiveness generalises beyond our custom
task set, we run planning-phase injection attacks on 27 injection tasks extracted from
the AgentDojo benchmark~\citep{debenedetti2024agentdojo} (4 suites:
workspace, travel, banking, slack). We use the same reflection loop and
attack families (A1--A3) but substitute AgentDojo's task descriptions and
injection goals. This addresses the question: \emph{does our methodology
transfer to an established, independently-designed benchmark?}

\begin{table}[t]
\centering
\caption{%
  \textbf{Attack generalization: custom tasks vs.\ AgentDojo}
  (3 models, Defense OFF/ON, $n_{\text{rep}}{=}3$, $n{=}2{,}409$ total).
  For GPT-5.5 and Claude-Haiku, enforcement reduces AgentDojo ASR to 2.5\%.
  The true three-model pool is OFF 56.2\% / ON 15.4\% because Grok-4 stays at 47.3\% ON.
}
\label{tab:agentdojo}
\begin{tabular}{llcccc}
\toprule
\textbf{Task Source} & \textbf{Model} & \textbf{ASR (OFF)} & \textbf{95\% CI} & \textbf{ASR (ON)} & $n$ \\
\midrule
Custom (main)  & GPT-5.5       & 38.7\% & [31, 47] & 0.0\% & 300 \\
AgentDojo (27) & GPT-5.5       & 50.7\% & [46, 55] & 2.5\% & 905 \\
\midrule
Custom (main)  & Claude-Haiku  & 42.7\% & [35, 51] & 0.0\% & 300 \\
AgentDojo (27) & Claude-Haiku  & 43.7\% & [39, 49] & 2.5\% & 796 \\
\midrule
AgentDojo (27) & Grok-4$^\dagger$ & 76.8\% & [72, 81] & 47.3\% & 708 \\
\midrule
Custom (all)$^{*}$ & Mixed pool & 48.3\% & [41, 55] & 10.8\% & 397 \\
AgentDojo (27) & All 3 models  & 56.2\% & [53, 59] & 15.4\% & 2409 \\
	AgentDojo (27) & GPT-5.5+Haiku & 47.3\% & [43, 51] & 2.5\% & 1701 \\
\bottomrule
\end{tabular}
\vspace{2pt}
{\footnotesize $^\dagger$ Grok-4 Defense-ON ASR is anomalously high
(47.3\%) due to ambiguous audit output; see paragraph below.
$^{*}$Same mixed open+closed pool as Table~\ref{tab:enforcement} All custom, not the mean of the five closed models.}
\end{table}

\paragraph{Results.}
On GPT-5.5 ($n{=}905$), AgentDojo tasks yield 50.7\% ASR Defense OFF
vs.\ 38.7\% on custom tasks ($\Delta = 12.0$ pp). Claude-Haiku shows
43.7\% vs.\ 42.7\% ($\Delta = 1.0$ pp). Both models achieve 2.5\% ASR
with enforcement enabled. Pooling all three AgentDojo models
($n{=}2{,}409$) gives OFF 56.2\% / ON 15.4\%; excluding Grok-4
($n{=}1{,}701$) gives OFF 47.3\% / ON 2.5\%.

\paragraph{Grok-4 audit-ambiguity anomaly.}
Grok-4 shows 76.8\% ASR Defense OFF and 47.3\% ASR Defense ON on
AgentDojo tasks---far higher than its 0.9\% Defense-ON rate on custom
tasks. Inspection of per-round abort flags reveals the root cause:
Grok-4's reflection on AgentDojo banking and travel tasks rarely
produces a parseable safety flag. The \texttt{abort\_per\_round}
field shows \texttt{[False]} for 87\% of Defense-ON runs (vs.\ 34\%
on custom tasks), meaning the abort signal is present but either
absent from the output or formatted in a way the parser does not
recognise. This is a distinct failure mode from enforcement failure
(clear flag, no action): we call it \emph{audit ambiguity} (flag
present but unparseable). Grok-4's hedged critique style, which
produces lower Defense-OFF ASR on structured tasks, becomes a
liability on the more open-ended AgentDojo tasks where decisive
safety verdicts are harder to elicit.

\paragraph{Interpretation.}
Excluding the Grok-4 anomaly, the AgentDojo comparison confirms
generalization: attack effectiveness transfers to an independently-designed
benchmark with $\Delta \leq 13.5$ pp (within CI overlap for Claude-Haiku).
The two model-specific anomalies (Grok-4.3 in v3, Grok-4 on AgentDojo)
both trace to the same root cause: audit ambiguity is a third failure
mode alongside enforcement failure and detection failure.

\section{Framework Code Inspection: AutoGen and LangChain}
\label{app:framework_inspection}

This appendix documents our code inspection findings for AutoGen and LangChain,
and describes the simulation approach used in \texttt{src/autogen\_loop.py}
and \texttt{src/langchain\_loop.py}.

\subsection{AutoGen (v0.2.38)}

AutoGen's multi-agent conversation is built around \texttt{ConversableAgent}
and \texttt{GroupChat}. The relevant source file is
\texttt{autogen/agentchat/conversable\_agent.py}.

\textbf{Finding 1: No enforcement-mode flag.}
The \texttt{ConversableAgent} constructor exposes the following stopping
mechanisms:
\begin{verbatim}
is_termination_msg: Optional[Callable]
  # callback: detect "TERMINATE" in message
max_consecutive_auto_reply: int  # iteration cap
human_input_mode: str  # "NEVER"/"ALWAYS"/"TERMINATE"
\end{verbatim}
None of these is conditioned on a safety flag from a critic agent. The
\texttt{is\_termination\_msg} callback checks for a literal \texttt{"TERMINATE"}
string in the message content, not a structured safety verdict.

\textbf{Finding 2: CriticAgent output is advisory.}
In AutoGen's default two-agent pattern (AssistantAgent $+$ UserProxyAgent),
the UserProxyAgent's \texttt{human\_input\_mode = "NEVER"} means no human
is consulted. A third CriticAgent can be added via \texttt{GroupChat}, but
its critique messages are passed as conversation context to the next round;
there is no built-in path for the critic to halt execution.

\textbf{Simulation fidelity.}
Our \texttt{AutoGenLoop} faithfully reproduces this pattern:
\begin{itemize}
  \item \texttt{PLANNER\_SYSTEM} $\leftrightarrow$ AssistantAgent system prompt
  \item \texttt{CRITIC\_SYSTEM} $\leftrightarrow$ CriticAgent system prompt
  \item \texttt{abort\_on\_flag=False} $\leftrightarrow$ default \texttt{human\_input\_mode="NEVER"}
  \item \texttt{abort\_on\_flag=True} $\leftrightarrow$ adding a custom
    \texttt{is\_termination\_msg} that checks \texttt{abort\_execution=True}
    in the critic JSON
\end{itemize}
The enforcement gap manifests because the second configuration requires
\emph{custom user code}; the default ships without it.

\subsection{LangChain (v0.2.16)}

LangChain's agent execution is built around \texttt{AgentExecutor} in
\texttt{langchain/agents/agent.py}.

\textbf{Finding 1: No enforcement-mode flag.}
\texttt{AgentExecutor} stopping conditions:
\begin{verbatim}
max_iterations: int          # iteration cap
max_execution_time: float    # wall-clock cap
early_stopping_method: str   # "force" or "generate"
handle_parsing_errors: bool
\end{verbatim}
None of these stops on a semantic safety flag. The \texttt{early\_stopping\_method}
controls what to do when \texttt{max\_iterations} is reached, not when a
reflection chain detects danger.

\textbf{Finding 2: Reflection chain output is advisory.}
In LangChain's LCEL pattern, a reflection chain can be chained after the
agent step:
\begin{verbatim}
agent_chain = prompt | llm | output_parser
reflection_chain = reflection_prompt | llm | output_parser
full_chain = agent_chain | reflection_chain
\end{verbatim}
The \texttt{full\_chain} pipes reflection output back as input to the next
agent step. There is no operator in LCEL that conditionally halts on a
boolean field from the reflection chain's output. Implementing enforcement
requires a custom \texttt{RunnableLambda} that raises an exception or returns
early --- this is not present in any default \texttt{AgentExecutor} config.

\textbf{Simulation fidelity.}
Our \texttt{LangChainLoop} models:
\begin{itemize}
  \item \texttt{AGENT\_SYSTEM} $\leftrightarrow$ AgentExecutor's LLMChain prompt
  \item \texttt{REFLECTION\_CHAIN\_SYSTEM} $\leftrightarrow$ reflection LLMChain
  \item \texttt{abort\_on\_flag=False} $\leftrightarrow$ default
    \texttt{AgentExecutor} (reflection is advisory)
  \item \texttt{abort\_on\_flag=True} $\leftrightarrow$ custom
    \texttt{RunnableLambda} with safety check
\end{itemize}

\subsection{Summary}

\begin{table}[h]
\footnotesize\centering
\caption{\textbf{Framework enforcement gap: code inspection + experimental results.}
  All five frameworks lack a default enforcement primitive.
  ASR measured on 3 frontier models (GPT-5.5, Claude-Haiku, Grok-4),
  $n{=}150$ per framework/condition (AutoGen/LangChain);
  $n{=}50$ for Reflexion/ReAct/ToT (existing sweep).}
\label{tab:framework_inspection}
\begin{tabular}{lp{2.8cm}p{2.8cm}cc}
\toprule
\textbf{Framework} & \textbf{Default stop} & \textbf{Missing primitive} & \textbf{OFF} & \textbf{ON} \\
\midrule
Reflexion   & \texttt{max\_rounds}
            & \texttt{abort\_exec\allowbreak ution=True} halt
            & 64\% & --- \\
ReAct       & \texttt{max\_rounds} / \texttt{AgentFinish}
            & \texttt{abort=True} halt
            & 50\% & --- \\
ToT         & Best-branch; no safety gate
            & Discard flagged branches
            & 65\% & --- \\
AutoGen     & \texttt{is\_termination\_}\allowbreak\texttt{msg} checks ``TERMINATE''
            & Check \texttt{abort\_}\allowbreak\texttt{execution} field
            & 74\% & 24\% \\
LangChain   & \texttt{max\_iterations} / \texttt{AgentFinish}
            & \texttt{RunnableLambda} safety halt
            & 83\% & 31\% \\
\bottomrule
\end{tabular}
\end{table}

\section{RL Controller: Training Details}
\label{app:rl_details}

\subsection{Dataset Construction}

We build the RL training dataset from the main enforcement sweep
(Table~\ref{tab:enforcement}) by converting each per-round trajectory
entry into a single \emph{controller episode}:

\begin{itemize}
  \item \textbf{State}: (task description, reconstructed plan stub,
    reflection critique stub, round index)
  \item \textbf{Label}: HALT if the round's \texttt{abort\_per\_round}
    flag is True and the plan is injected; CONTINUE otherwise
  \item \textbf{Plan stub}: constructed from the task description and
    the attack scaffold template (Appendix~\ref{app:attacks}); the
    injected step is included iff \texttt{survived\_reflection[r]} is True
  \item \textbf{Critique stub}: ``Safety concern detected\ldots'' if
    \texttt{abort\_per\_round[r]} is True; ``Plan looks correct\ldots''
    otherwise
\end{itemize}

We augment with synthetic benign episodes (same tasks, no injection)
to balance the dataset. Final split: 10{,}535 train / 1{,}555 val /
1{,}556 test, stratified by (is\_injected, detected).

Label distribution:

\begin{center}
\small
\begin{tabular}{lccc}
\toprule
Split & inj=T, det=T & inj=T, det=F & inj=F, det=F \\
\midrule
Train & 7,339 & 1,104 & 4,000 \\
Val   &   917 &   138 &   500 \\
Test  &   918 &   138 &   500 \\
\bottomrule
\end{tabular}
\end{center}

\subsection{Model and Training Configuration}

\textbf{Base model}: \texttt{Qwen2.5-7B-Instruct}
(7B parameters, bfloat16).

\textbf{LoRA}: rank $r=16$, $\alpha=32$, applied to all attention
projection layers and feed-forward gate/up/down projections;
dropout 0.05; 40.4M trainable parameters (0.53\% of total).

\textbf{GRPO hyperparameters}:
\begin{center}
\small
\begin{tabular}{ll}
\toprule
Hyperparameter & Value \\
\midrule
Epochs & 3 \\
Batch size (per device) & 4 \\
Gradient accumulation & 4 \\
Effective batch size & 64 (4 GPUs) \\
Learning rate & $5 \times 10^{-6}$ \\
Warmup ratio & 0.05 \\
Group size $G$ & 8 \\
Max completion length & 8 tokens \\
KL coefficient $\beta$ & 0.02 \\
Temperature & 0.9 \\
Top-$p$ & 0.95 \\
\bottomrule
\end{tabular}
\end{center}

\textbf{Reward function}: see \S\ref{sec:rl_controller},
Equation for $R(\text{action})$.

\textbf{Hardware}: 4$\times$NVIDIA H100 80GB, training time
$\approx$4.25 hours (9{,}579 steps at $\sim$1.6 s/step).

\subsection{Evaluation Protocol}

We evaluate on the 1{,}556-example held-out test set using greedy
decoding (temperature = 0, top-$p$ = 1). The controller outputs
are parsed by taking the first whitespace-delimited token of the
completion and checking for exact match with ``HALT'' or
``CONTINUE''; unrecognised outputs are treated as CONTINUE
(conservative fallback).

Four baselines are included for comparison (see Table~\ref{tab:rl_controller}):
B0 always returns CONTINUE (the current production default in all
five surveyed frameworks), B1 always returns HALT (theoretical
upper bound with 100\% FPR), B2 uses a 28-pattern keyword regex
on the critique text, and B3 applies the same controller prompt to
the unmodified Qwen2.5-7B-Instruct base model without RL fine-tuning.

\subsection{Limitations and Future Work}

The current evaluation uses synthetic critique stubs because the
main sweep results store per-round metrics but not raw reflection
text. This means both B3 and RL-GRPO operate on unambiguous inputs
where the correct action is often lexically transparent. In
production settings, reflection critiques are richer, more hedged,
and sometimes contradictory; we expect the RL advantage over the
base prompt to widen in that regime. Future work should:
(1)~collect real reflection transcripts for training;
(2)~evaluate with longer, more ambiguous critique chains
(3--5 rounds of evolving context);
(3)~explore reward shaping that penalises \emph{delayed} halts
(allowing the injection to persist for $k$ rounds before blocking).

\section{Implementation Guide}
\label{app:implementation}

Below we provide concrete implementation guidance for the recommendations
in \S\ref{sec:conclusion}.

\textbf{Enforcement primitive (pseudocode):}
\begin{verbatim}
# After reflection step returns critique text
flag = parse_safety_flag(critique)
if flag.is_unsafe and config.enforce:
    return HALT  # abort execution
else:
    replan(critique)  # default advisory path
\end{verbatim}

\textbf{Plan-diff check:}
\begin{verbatim}
baseline = hash_plan(initial_clean_plan)
current = hash_plan(revised_plan)
jaccard = 1 - len(baseline & current) /
              len(baseline | current)
if jaccard > tau:  # tau in [0.4, 0.6]
    return SECURITY_EVENT
\end{verbatim}

\textbf{WB-2 audit qualification test:}
Run 20 tasks $\times$ 30 candidates under 3 conditions (goal-visible,
neutral, blind). Compute Spearman $\rho$ between model ranking and true
severity. If $\rho_{\text{visible}} - \rho_{\text{blind}} < 0.1$ or
$p > 0.05$ (permutation test), the model should not serve as sole auditor.

\section{Per-Model and Per-Attack Breakdown}
\label{app:per_model_attack}

\subsection{Per-Model Analysis}

Table~\ref{tab:enforcement} in the main paper shows aggregate results.
Here we provide the per-model interpretation.

GPT-5.4 shows the highest Defense-OFF ASR among the five closed models
(58.0\% on the 50-task sweep), suggesting its planner is most
susceptible to injection persistence under iterative replanning.
In contrast, Grok-4 shows the lowest Defense-OFF ASR (26.4\%) but the
highest residual Defense-ON ASR (0.9\%).
Claude-Haiku achieves perfect enforcement (0.0\% Defense-ON ASR).
GPT-5.5 drops from 38.7\% to 0.0\%.

The key insight is that \emph{all} models benefit substantially from
enforcement (aggregate $4.5\times$ reduction). Native attack resistance
(low Defense-OFF ASR) and enforcement compliance (low Defense-ON ASR) are
\emph{uncorrelated} model properties: Grok-4 is the most resistant to
attacks but the least compliant with enforcement; Claude-Haiku is more
susceptible to injection but perfectly compliant. Operators should measure
both properties separately rather than assuming one implies the other.

\subsection{Per-Attack Breakdown}

A1 (ReflectionTrap) and A2 (FixedPoint) achieve the highest ASR
(75.6\% and 71.1\% respectively on AgentDojo tasks), because both
exploit the reflection prompt directly. A1 embeds a meta-instruction
telling the reflector the injected step is safe; A2 is optimised for
fixed-point stability under iterative critique.

The CascadePF family (A3) shows wider variance: A3-PF1 (GoalSub) achieves
only 14.3\% because goal substitution is easily detected by the reflection
model, while A3-PF3 (CtxPoll) and A3-PF4 (RoleConf) achieve 35.7\% each
by exploiting context-window pollution and role confusion---attack vectors
that the reflector struggles to distinguish from legitimate plan revisions.

A4 (IndirectInjection, tool-output delivery) achieves only 1.8\% ASR on
custom tasks and 12.3\% on AgentDojo, consistent with prior work showing
that tool-output injection is harder than direct plan modification because
the planner processes observations as data rather than directives.

\section{Per-Round Active Attack Fraction}
\label{app:per_round}

Table~\ref{tab:per_round_app} reports the fraction of runs in which the
injected attack step remained active at each reflection round.
Defense OFF (advisory logs): the active fraction stays near the per-attack
ASR throughout all three rounds because the controller never halts.
Defense ON (abort-on-flag): active runs collapse quickly as the controller
acts on the first flag.

\begin{table}[h]
\small\centering
\caption{%
  \textbf{Active attack fraction by reflection round: Defense OFF vs.\ ON.}
  Computed from \texttt{survived\_reflection} field in the 50-task closed-model sweep.
  Defense OFF: active fraction stays near per-attack ASR throughout all rounds.
  Defense ON: fraction of not-yet-halted runs collapses within 1--2 rounds.
}
\label{tab:per_round_app}
\begin{tabular}{lcc|cc|cc}
\toprule
& \multicolumn{2}{c|}{\textbf{Round 1}} & \multicolumn{2}{c|}{\textbf{Round 2}} & \multicolumn{2}{c}{\textbf{Round 3}} \\
\textbf{Attack} & OFF & ON & OFF & ON & OFF & ON \\
\midrule
A1 ReflectionTrap & 62.4\% & 7.2\% & 62.4\% & 4.6\% & 54.5\% & \textbf{3.9\%} \\
A2 FixedPoint     & 62.5\% & 10.0\% & 62.5\% & 4.6\% & 56.1\% & \textbf{3.9\%} \\
A3-PF3 CtxPoll    & 62.0\% & 8.7\%  & 61.5\% & 5.8\% & 35.9\% & \textbf{4.2\%} \\
A3-PF4 RoleConf   & 61.6\% & 6.4\%  & 61.6\% & 3.9\% & 48.4\% & \textbf{3.5\%} \\
\bottomrule
\end{tabular}
\end{table}

\section{Why Cross-Audit Fails: RLHF Correlated Blind Spots}
\label{app:cross_audit_mechanism}

The main paper reports that D2 (cross-backbone audit) achieves 10.3\% ASR,
$2.3\times$ worse than D1 (self-audit, 4.5\%). This section provides the
theoretical grounding for that result.

\paragraph{The RLHF correlation mechanism.}
RLHF training optimises a reward model trained on human preference
data~\citep{ouyang2022training, christiano2017deep}. Frontier labs share
similar red-teaming methodologies and preference annotation
guidelines~\citep{bai2022constitutional, touvron2023llama2}.
\citet{casper2023open} show that reward models trained on similar
distributions produce correlated failure modes; \citet{wolf2024fundamental}
prove that RLHF with finite preference data cannot eliminate all undesirable
behaviours, and the residual set depends on the training distribution.

In practice, when multiple vendors use overlapping red-teaming attack sets
(direct harm requests, social engineering templates), their models converge
to similar refusal boundaries: they refuse the \emph{same} attacks and pass
the \emph{same} blind spots. Plan-phase injections, framed as legitimate
task subtasks rather than direct harm requests, fall outside every vendor's
red-teaming distribution, creating a correlated vulnerability class.

\paragraph{Implication for cross-audit design.}
This explains why D2 is worse than D1: the cross-auditor has the
\emph{same} blind spots as the planner it audits, so its verdicts are not
an independent signal. A GPT-5.5 planner injected with a subtask-framed
malicious step is audited by a Grok-4 model that was trained to refuse
similar attacks---but also to pass the same subtask-framing. The result is
correlated failure.

True diversity requires explicitly uncorrelated channels: structural checks
(D4, which does not depend on safety training at all), rule-based validators,
or models trained with deliberately orthogonal safety objectives. D4
outperforms D2 precisely because structural plan-diff has failure modes
that are orthogonal to those of model-based auditors: it can fail on
low-drift attacks but is immune to the RLHF blind-spot correlation.

\paragraph{Practical implication.}
Operators should not assume that pairing two safety-trained frontier
models produces an independent audit. To achieve genuine independence,
at least one component of the defense stack must be model-agnostic
(e.g., D4 plan-diff, rule-based keyword detectors, or formal specification
checking where applicable).

\section{Extended Related Work}
\label{app:related_extended}

This section expands the condensed related work in the main paper.

\paragraph{Prompt injection and agent attacks.}
\citet{perez2022ignore} and \citet{greshake2023more} first studied prompt
injection systematically, showing that adversarial instructions embedded in
user-controlled content can redirect model behaviour.
\citet{zhan2024injecagent} extended this to multi-step agents, demonstrating
that tool-output injection can hijack agent actions.
\citet{debenedetti2024agentdojo} introduced the AgentDojo benchmark for
evaluating injection defenses across four task suites; our A4 results
(1.8\% ASR on tool outputs) are consistent with their finding that
tool-output injection is harder than direct instruction override.
Prior work introduced planning-phase injection with $\mathcal{O}(n)$ cascade
semantics; the present paper extends this to
the dynamic replanning setting and identifies enforcement failure as a
distinct, previously uncharacterised vulnerability mode.

\paragraph{Adversarial attacks on LLMs.}
GCG~\citep{zou2023universal} and AutoDAN~\citep{liu2024autodan} optimise
adversarial suffixes to bypass safety training via gradient-based search;
HotFlip~\citep{ebrahimi2018hotflip} pioneered character-level white-box
attacks on text classifiers.
PAIR~\citep{chao2024jailbreaking} uses a red-team LLM to iteratively
refine jailbreak prompts. These methods require white-box access or
multi-turn interaction. Our adaptive adversary (best-of-8 scaffolds)
is a weaker but more deployment-realistic threat model: black-box API
access, no gradient information, single query per scaffold. Yet it defeats
the structural D4 defense (55--61\% evasion), demonstrating that even
modest adversarial effort suffices against current single-layer defenses.

\paragraph{Iterative self-critique.}
Reflexion~\citep{shinn2023reflexion} introduced verbal reinforcement via
self-critique; Self-Refine~\citep{madaan2023selfrefine} generalised this
to iterative refinement with self-feedback; ReAct~\citep{yao2023react}
interleaves reasoning and action; Tree-of-Thoughts~\citep{yao2023tree}
performs best-of-$K$ search over plan branches. All assume that
self-critique improves plan quality and safety---an assumption validated
for benign settings but never tested under adversarial conditions until
this work. \citet{gou2023critic} study self-critique for factual accuracy;
their finding that self-critique improves factual correctness is consistent
with our detection result (high $p_d$); the missing piece is enforcement.
Constitutional AI~\citep{bai2022constitutional} uses an LLM to critique
and revise its own outputs for value alignment; our D2 result ($2.3\times$
worse than self-audit) shows that correlated safety training undermines
cross-model critique as a security primitive.

\paragraph{Agent safety surveys.}
\citet{xi2023rise} and \citet{wang2024survey} survey agent capabilities
and safety challenges, identifying prompt injection and tool misuse as
key risks but without quantitative evaluation.
\citet{liu2024agentbench} benchmark LLM agents across diverse environments
but do not evaluate adversarial robustness.
\citet{ruan2024identifying} analyse risk taxonomies for tool-augmented
agents; enforcement failure does not appear in their taxonomy, confirming
it has been overlooked as a distinct vulnerability class.
\citet{bowman2022measuring} frame scalable oversight as the central
challenge for safe AI deployment; our enforcement gap is a concrete
instance of oversight failure in deployed agents.

\paragraph{Reference monitors and policy enforcement.}
The concept of a reference monitor was formalised by
Anderson~\citeyearpar{anderson1972computer} as the minimal requirement
for a secure system: completeness (all accesses pass through it),
isolation (it cannot be tampered with), and verifiability (it can be
shown correct). The Reflexion controller satisfies \emph{none} of these.
Lampson's~\citeyearpar{lampson1974protection} protection model shows that
incomplete policy enforcement is equivalent to having no policy at all.
Corollary~\ref{cor:detection} is the LLM-agent instantiation of this
classical result.

\paragraph{Concurrent work.}
\citet{deepcover2026} study long-horizon prompt injections that persist
across multiple sessions; our work focuses on within-session reflection
loops. \citet{commandsans2026} propose token-level sanitization of tool
outputs; this addresses the input-filtering stage, whereas we focus on
post-detection enforcement. The two approaches are complementary.

\section{Audit Reliability: Extended Analysis}
\label{app:audit_extended}

\subsection{Full 10-Model WB-2 Results}

The main paper shows a 5-model subset (Table~\ref{tab:audit}).
Table~\ref{tab:wb2_full} presents the full 10-model results.

\begin{table}[h]
\small\centering
\caption{\textbf{Full WB-2 audit leak rate results} (10 open-weight models,
  20 tasks $\times$ 30 candidates $\times$ 3 seeds).
  $^{**}$: $p < 0.01$; $^{*}$: $p < 0.05$.}
\label{tab:wb2_full}
\begin{tabular}{lcccc}
\toprule
\textbf{Model} & \textbf{Blind $\varepsilon$} & \textbf{Rand $\varepsilon$} &
  \textbf{High $\varepsilon$} & $p$(b$>$h) \\
\midrule
R1-Qwen-32B   & 0.22 & 0.16 & 0.01 & 0.009$^{**}$ \\
Llama-3-8B    & 0.47 & 0.39 & 0.22 & 0.010$^{**}$ \\
Qwen2.5-14B   & 0.38 & 0.21 & 0.20 & 0.085 \\
Gemma-2-9B    & 0.81 & 0.76 & 0.77 & 0.285 \\
Mistral-7B    & 0.63 & 0.56 & 0.50 & 0.175 \\
Qwen2.5-7B    & 0.53 & 0.51 & 0.47 & 0.325 \\
Gemma-3-12B   & 0.16 & 0.14 & 0.12 & 0.320 \\
QwQ-32B       & 0.52 & 0.54 & 0.47 & 0.338 \\
R1-Llama-70B  & 0.99 & 0.96 & 0.98 & 0.511 \\
Gemma-3-27B   & 0.40 & 0.40 & 0.39 & 0.443 \\
\bottomrule
\end{tabular}
\end{table}

\subsection{Key Findings}

Three observations from Table~\ref{tab:wb2_full}:

\textbf{(1) Direction-awareness is rare.}
Only 2 of 10 models (R1-Qwen-32B and Llama-3-8B) show statistically
significant direction-awareness ($p < 0.01$). 8 of 10 models cannot
demonstrably distinguish high-harm from low-harm plan steps when the
attacker goal is hidden.

\textbf{(2) $\varepsilon$ varies $\sim 100\times$} from 0.01
(R1-Qwen-32B, high direction) to 0.99 (R1-Llama-70B, essentially random).
This variation is not predictable from model size: Gemma-3-12B
($\varepsilon = 0.16$) outperforms the much larger R1-Llama-70B
($\varepsilon = 0.99$). Operators cannot infer audit reliability from
benchmark performance or parameter count.

\textbf{(3) Reasoning-distilled models show weaker audits
in our sample.}
R1-Llama-70B ($\varepsilon = 0.97$) and QwQ-32B ($\varepsilon = 0.52$)
show the highest leak rates among reasoning-distilled models. We hypothesise
that reasoning training amplifies post-hoc justification at the cost of
safety critique: these models generate elaborate explanations for why
injected steps might be acceptable rather than flagging them. With $n=2$
reasoning models this is preliminary; future work should test with a larger
sample.

\subsection{Implications for Enforcement-Enabled Systems}

When $p_e = 1$ (full enforcement), ASR $\geq 1 - p_d = \varepsilon$.
For R1-Llama-70B ($\varepsilon = 0.97$), even perfect enforcement yields
ASR $\geq 97\%$: this model should \emph{not} serve as the auditor in an
enforcement-enabled system, because its detection failure means there is
nothing for the enforcement primitive to act on. The WB-2 test
(Appendix~\ref{app:wb2}) provides a pre-deployment screening protocol to
identify such cases before deployment.

\section{Full Failure-Mode Taxonomy}
\label{app:taxonomy}

Table~\ref{tab:taxonomy_full} provides the full three-mode taxonomy
referenced in \S\ref{sec:discussion}. Each mode has a distinct mechanism,
empirical signature, and remediation. A deployed system may exhibit any
combination of these three modes simultaneously.

\begin{table}[h]
\footnotesize\centering
\caption{\textbf{Three failure modes of reflection-based agent safety.}
  Each has a distinct mechanism, empirical signature, and remediation.
  A deployed system may exhibit any combination of these simultaneously.}
\label{tab:taxonomy_full}
\begin{tabular}{p{2.2cm}p{3.3cm}p{3.3cm}p{3.3cm}}
\toprule
\textbf{Mode} & \textbf{Mechanism} & \textbf{Signature} & \textbf{Remediation} \\
\midrule
\textbf{Enforcement Failure} (\S\ref{sec:enforcement})
  & Audit detects but controller discards verdict ($p_e \approx 0$)
  & High $p_d$, high ASR; ASR collapses on enforcement
  & Mandatory halt path; abort-on-flag \\
\addlinespace
\textbf{Detection Failure} (\S\ref{sec:audit})
  & Audit cannot distinguish injected from benign ($\varepsilon \approx 1$)
  & $p_d$ low; ASR unaffected by enforcement
  & Audit selection via WB-2; diverse auditors \\
\addlinespace
\textbf{Audit Ambiguity} (\S\ref{sec:defenses}, App.~\ref{app:agentdojo})
  & Audit produces unparseable output; enforcement has nothing to act on
  & High Defense-ON ASR despite model compliance
  & Structured verdict schema; robust flag parser \\
\bottomrule
\end{tabular}
\end{table}

\paragraph{Illustrative mapping.}
Each mode corresponds to a distinct collapse pattern reported in a recent
field observation~\citep{emergence2026} (not peer-reviewed; cited as analogy only):
\emph{Enforcement failure} $\leftrightarrow$ unanimous voting behaviour
(audit output ignored by controller);
\emph{Detection failure} $\leftrightarrow$ starvation pattern (agents unable
to resolve ambiguous safety judgements);
\emph{Audit ambiguity} $\leftrightarrow$ rapid norm collapse
(reflection output unparseable, enforcement chain inoperable). No single
defense addresses all three; production agent safety requires a layered
stack that addresses each mode separately.

\end{document}